\documentclass[letterpaper]{article} 
\usepackage{aaai23}  
\usepackage{times}  
\usepackage{helvet}  
\usepackage{courier}  
\usepackage[hyphens]{url}  
\usepackage{graphicx} 
\usepackage{natbib}  

\usepackage{caption} 
\usepackage{algorithm}
\usepackage{algorithmic}

\usepackage{newfloat}
\usepackage{listings}
\DeclareCaptionStyle{ruled}{labelfont=normalfont,labelsep=colon,strut=off} 
\floatstyle{ruled}
\newfloat{listing}{tb}{lst}{}
\floatname{listing}{Listing}
\usepackage{algorithm}
\usepackage{algorithmic}
\usepackage{booktabs}
\usepackage{amsmath}
\usepackage{amssymb}
\usepackage{mathtools}
\usepackage{amsthm}

\newtheorem{proposition}{Proposition}

\newtheorem{definition}{Definition}
\usepackage{enumitem}

\usepackage{bm}
\usepackage{longtable}
\usepackage{array}
\usepackage{multirow}
\usepackage{diagbox}
\usepackage{subcaption}

\title{Evaluating Model-free Reinforcement Learning toward Safety-critical Tasks}
\author{
    Linrui Zhang\textsuperscript{\rm 1}\thanks{This work was done during Linrui Zhang's internship at JD Explore Academy, JD.com, Inc.}, 
    Qin Zhang\textsuperscript{\rm 1}, 
    Li Shen\textsuperscript{\rm 2}$^\dag$\footnote{Correspondence to: Li Shen and Xueqian Wang},  
    Bo Yuan\textsuperscript{\rm 3}, 
    Xueqian Wang\textsuperscript{\rm 1}$^\dag$, 
    Dacheng Tao\textsuperscript{\rm 2}
}

\affiliations{
    \textsuperscript{\rm 1} Tsinghua Shenzhen International Graduate School, Tsinghua University\\
    \textsuperscript{\rm 2} JD Explore Academy, JD.com, Inc.\\
    \textsuperscript{\rm 3} Qianyuan Institute of Sciences\\
    \{zlr20,zhangqin21\}@mails.tsinghua.edu.cn, \{mathshenli,dacheng.tao\}@gmail.com, \\
    boyuan@ieee.org, wang.xq@sz.tsinghua.edu.cn

}

\usepackage{bibentry}

\begin{document}

\maketitle

\begin{abstract}
Safety comes first in many real-world applications involving autonomous agents. Despite a large number of reinforcement learning (RL) methods focusing on safety-critical tasks, there is still a lack of high-quality evaluation of those algorithms that adheres to safety constraints at each decision step under complex and unknown dynamics. In this paper, we revisit prior work in this scope from the perspective of state-wise safe RL and categorize them as projection-based, recovery-based, and optimization-based approaches, respectively.  Furthermore, we propose Unrolling Safety Layer (USL), a joint method that combines safety optimization and safety projection. This novel technique explicitly enforces hard constraints via the deep unrolling architecture and enjoys structural advantages in navigating the trade-off between reward improvement and constraint satisfaction. To facilitate further research in this area,  we reproduce related algorithms in a unified pipeline and incorporate them into SafeRL-Kit, a toolkit that provides off-the-shelf interfaces and evaluation utilities for safety-critical tasks. We then perform a comparative study of the involved algorithms on six benchmarks ranging from robotic control to autonomous driving. The empirical results provide an insight into their applicability and robustness in learning zero-cost-return policies without task-dependent handcrafting. The project page is available at \url{https://sites.google.com/view/saferlkit}.
\end{abstract}

\section{Introduction}
Model-free reinforcement learning (RL) has achieved superhuman performance on many decision-making problems~\citep{mnih2015human,vinyals2019grandmaster}. Typically, the agent learns from trial and error and requires minimal prior knowledge of the environment. Such a paradigm features significant advantages in mastering essential skills for complex systems, but concerns about the systematic safety limit the extent of adoption of model-free RL in real-world applications, such as human-robot collaboration~\citep{villani2018survey} and autonomous driving~\citep{kiran2021deep}.

Penalizing unsafe transitions via the reward function is straightforward but sometimes cumbersome to navigate the trade-off between performance and safety. A trivial penalty term may fail to obtain a risk-averse policy, whereas an excessive punishment may make the agent too conservative to explore the environment. Alternatively, incorporating safety into RL via constraints~\cite{altman1999constrained} is  widely adopted since the strength of constraints reflects the human-specified safety requirement, and the agent is desired to optimize its behavior within the constrained policy search space. 

In this paper, we explore model-free reinforcement learning methods that adhere to state-wise safety constraints. To better understand how this study is developed, two points deserve further clarification. First, our work is aimed at learning a stationary safe policy under the general model-free settings, instead of refusing any safety violations during the training. The latter is more related to optimal control and relies on the known dynamic model~\cite{cheng2019end} or a carefully designed energy function~\cite{zhao2021model}. Second, we focus on the state-wise constraint at every decision-making step and demonstrate that this type of constraint is more strict and practical toward safety-critical tasks theoretically and empirically. Our contributions in this paper are summarized as follows:

\begin{enumerate}
\item We revisit model-free RL following state-wise safety constraints and present SafeRL-Kit, a toolkit that implements prior work in this scope under a unified off-policy framework. Specifically, SafeRL-Kit contains projection-based Safety Layer~\cite{dalal2018safe}, recovery-based Recovery RL~\cite{thananjeyan2021recovery}, optimization-based Off-policy Lagrangian~\cite{ha2020learning}, Feasible Actor-Critic~\cite{ma2021feasible}, and the new method proposed in this paper.
\item We propose Unrolling Safety Layer (USL), a novel approach that combines safety projection and safety optimization. USL unrolls gradient-based corrections to the jointly optimized actor-network and thus explicitly enforces the constraints. The proposed method is simple-yet-effective and outperforms state-of-the-art algorithms in learning risk-averse policies. 
\item We perform a comparative study based on SafeRL-Kit and evaluate the related algorithms on six different tasks. We further demonstrate their applicability and robustness in safety-critical tasks  with the universal binary cost indicator and a constant constraint threshold.
\end{enumerate}

\section{Related Work}
\paragraph{Safe RL Algorithms.}
Safe RL optimizes policies under episodic or instantaneous constraints. 
The most common approach to solving the episodic constraint is Lagrangian relaxation (Chow et al. 2017; Tessler at al. 2017; Stooke et al. 2020). Other works~\citep{achiam2017constrained,yang2020projection} approximate the constrained policy iteration with a quadratic constrained optimization. Recently, first-order methods~\cite{liu2020ipo,ijcai2022-520,yang2022constrained} start to gain attractions as the objective is efficient to optimize and easy to handle multiple constraints.
For the instantaneous constraint, Lagrangian relaxation is also a feasible solution~\cite{bohez2019value}. \citet{dalal2018safe} perform quadratic programming to project actions back to the safe set. Other works model the instantaneous cost as Gaussian Processes and plan in the safety-proven neighboring states~\citep{wachi2018safe,wachi2020safe}. 
In this paper, we formulate safe RL following state-wise safety constraints, which are slightly different from the above genres.

\paragraph{Safe RL Benchmarks.}
There have already been some safety-critical benchmarks to evaluate the efficacy of safe RL methods, including traditional MuJoCo tasks~\cite{achiam2017constrained,zhang2020first}, navigation in the cluttered environment~\cite{ray2019benchmarking}, safe robotic control task~\cite{yuan2021safe} and safe autonomous driving~\cite{li2021metadrive,herman2021learn}. However, a comprehensive study on learning a zero-cost-return policy with model-free methods is still absent.

\section{Preliminaries}
This study lies in the context of constrained Markov Decision Process (CMDP)~\citep{altman1999constrained}, which extends standard MDP~\citep{sutton2018reinforcement} as a tuple $(\mathcal{S},\mathcal{A},\mathcal{P},\mathcal{R},\mathcal{C},\mu,\gamma)$. ${\mathcal{S}}$ and ${\mathcal{A}}$ denote the state space and the action space, respectively. ${\mathcal{P}}: {\mathcal{S}} \times \mathcal{A} \times  \mathcal{S} \mapsto [0,1]$ is the transition  probability  function to describe the dynamics of the system.  $\mathcal{R} : \mathcal{S} \times \mathcal{A} \mapsto \mathbb{R}$ is the reward function. $\mathcal{C} : \mathcal{S} \times \mathcal{A} \mapsto [0,+\infty] $ is the cost function and  reflects the safety violation. $\mu:\mathcal{S} \mapsto [0,1]$ is the initial state distribution, and $\gamma$ is the discount factor for future reward and cost.
A stationary policy $\pi : S \mapsto P(A)$ maps the given states to probability distributions over action space and the expected discounted return of the policy is $J_R(\pi) = \mathop{\mathbb{E}}_{\tau\sim \pi}\big [ \sum^\infty_{t=0}\gamma^t R(s_t,a_t)\big ],$ where $\tau\sim\pi$ accounts for the stochastic trajectory distribution sampled on $s_0 \sim \mu, a_t \sim \pi(\cdot | s_t), s_{t+1} \sim P(\cdot | s_t,a_t)$. The goal of safe RL is to find the optimal policy: 
\begin{equation}
{\max}_\pi J_R(\pi)\quad \mathrm{s.t.}\ \ \pi \text{ is feasible}.
\end{equation}

In a CMDP, the agent is typically constrained by the cost function in two ways. One is the \emph{Episodic Constraint}. This type of formulation requires the cost-return in the whole trajectory be within a certain threshold, namely $J_C(\pi) = \mathop{\mathbb{E}}_{\tau\sim \pi}\big [ \sum^\infty_{t=0}\gamma^t c_t\big ] \leq d$, which is suitable for total energy consumption, resource over-utilization, etc. The other is the \emph{Instantaneous Constraint}. This type of formulation requires the selected actions to enforce the constraint at every decision-making step, namely $\forall t,  \mathop{\mathbb{E}}_{\pi}[c_t|s_t] \leq \epsilon$, which is indispensable in accident and damage avoidance.

\section{Revisit RL toward Safety-critical Tasks}

In many safety-critical scenarios, the final policy is assumed to maintain the zero-cost return since any inadmissible behavior could lead to catastrophic failure in the execution. Prior constrained learning paradigms have fatal flaws under this premise. For the episode constraint, with a threshold $d$ close to $0$, the agent often either fails to improve policy or receives a cost-return more significant than $0$. The underlying reason is that such a constraint is easy to violate, especially when the time horizon contains thousands of steps. If the algorithm handles the episodic formula directly, it is cumbersome to identify ill actions as well as tweak the policy parameters and the corresponding action sequence. For another, the instantaneous constraint is tighter since it is a sufficient but not necessary condition for the episodic constraint with $\epsilon = (1-\gamma)d$. Nevertheless, it is problematic to directly enforce $C(s_t,a_t)\leq \epsilon$ at each single decision-making step in complicated dynamics, since some actions have a long-term effect on future visited states and the infeasible states might also be intractable to recover in a single time-step. 
A more reasonable solution is to prevent the current state from falling into the unsafe set in a certain planning span, which is inspired by model predictive control. Consequently, we define the long-term return for cost as $ Q^\pi_c(s_t,a_t) ={\mathbb{E}}_{\pi} \big [ \sum^\infty_{t'=t}\gamma^t c_{t'} | s_t, a_t\big]$, which is similar to $Q^\pi(s,a)$ for reward in standard RL by substituting $r_t$ with $c_t$. Next, we present 
the formal definitions related to state-wise safe reinforcement learning.
\begin{definition}[State-wise safety constraints] In the whole trajectory, the agent is required to adhere to the following long-term constraint at every visited state
\begin{equation}
    Q^\pi_c(s_t,a_t) \leq \delta, \forall t \geq 0.
    \label{sw-constraints}
\end{equation}
\end{definition}
\begin{definition}[Optimal state-wise safe policy] In any feasible state, the optimal action is
\begin{equation}
    a^* = \mathop{\arg\max}_{a} \big[Q^\pi(s,a) \big] \quad \mathrm{s.t.} \ \ Q^\pi_c(s,a) \leq \delta,
    \label{infer}
\end{equation}
\end{definition}

Using the cumulative constraint to enhance state-wise safety is not novel~\cite{srinivasan2020learning,ma2021feasible,yu2022reachability}. Nevertheless, the choice of $\delta$ is tricky since the relationship between~\eqref{sw-constraints} and the desired 
instantaneous constraint is not straightforward. In this paper, we give a theoretical bound of cost limit $\delta$ as follows. 

\begin{proposition}\label{bound}
 If $\delta \leq \epsilon\cdot\gamma^T $,  any policy $\pi(\cdot|s)$  satisfying \eqref{sw-constraints} fulfills $\mathop{\mathbb{E}}_{\pi}[c_t|s_t] \leq \epsilon$ within the planning span $T$.
\end{proposition}
\begin{proof}
It is proved by contradiction that if $\mathbb{E}_{\pi}\big[c_t|s_t] > \epsilon$ in any step $T' < T$, then $V^\pi_c(s_t)\geq \epsilon\cdot\gamma_c^{T'} >  \epsilon\cdot\gamma_c^T \geq \delta$.
\end{proof}

The above proposition, to some extent, alleviates the concern that decreasing $\delta$ leads to overly conservative policy. For example, if a racing car is able to slow down and avoid the obstacle in 20 steps ahead, setting $T > 20$ will not change the optimal sequence for $a^*_{t-t'}$, where $t' > 20$.
In our experiments, we keep $\delta=0.1$ across different safety-critical tasks, which equals the safe planning span of at least 100 steps with a universal binary cost indicator. Empirical results demonstrate that safe RL methods adhering to state-wise safety constraints are robust to this value.

In this paper, we explore model-free RL that adheres to state-wise safety constraints in continuous state-action spaces and unknown dynamics. We classify the most related approaches into the following three categories:

\paragraph{Safety Correction.}
This type of methods corrects the initial unsafe decision by projecting it back to the safe set. The projection can be constructed by the control barrier function~\citep{cheng2019end} with known dynamics, implicit safety index~\citep{zhao2021model} with hand-crafted energy function, or parametric linear model~\cite{dalal2018safe}  learned from past experiences. However, those approaches are sometimes under-performed regarding cumulative rewards since the correction only guarantees the feasibility but lacks the equivalence with optimality.

\paragraph{Safety Recovery.}

This type of methods is especially welcomed in the field of robotics~\cite{thananjeyan2021recovery,yang2022safe} and autonomous driving~\cite{chen2021safe}. The critical idea behind Recovery RL is introducing a dedicated policy that recovers unsafe states, whereas the task policy is trained by the standard RL to achieve the original goal. However, those approaches struggle for a rational recovery policy since it tends to be overly conservative in preventing risky exploration. Furthermore, the decisions between two policies may conflict with each other, which makes the agents stuck near the boundaries of safe regions easily.

\paragraph{Safety Optimization.}
This type of methods incorporates safety constraints into the RL objective and yields a constrained sequential optimization task.  
These approaches employ different optimization objectives to guide the updates of parametric policies, which can be tackled by Lagrangian relaxation~\cite{ha2020learning,ma2021feasible} or the penalty method~\cite{ijcai2022-520}. Unfortunately, the ``soft" loss function in the sample-based learning does not consistently enforce the ``hard" constraint in practice and barely leads to zero-cost-return policies even at convergence.

\section{Unrolling Safety Layer: A Novel Approach}
Orthogonal to existing algorithms, we propose a novel approach referred to as \textbf{U}nrolling \textbf{S}afety \textbf{L}ayer (USL) in this paper, which is inspired by the complementary of safety projection and safety optimization. For projection-based approaches, the correction (even if tractable) only enforces the feasibility but lacks the equivalence to the optimal maximum return. For optimization-based approaches, most of them tend to find the optimal solution to the constrained problem with the help of neural networks. However, the forward computing lacks explicit restrictions on the output actions, and thus the ``soft" loss function often fails to fully satisfy ``hard" constraints.
Recently, Deep Constraint Completion and Correction (DC3)~\cite{donti2021dc3} shows potential to achieve optimal objective values while preserving feasibility, which is of independent interest to general constrained problems. As illustrated in Figure~\ref{fig:usl}, we employ a similar joint architecture that combines safety optimization (serves as the first stage's approximate solver) and safety projection (serves as the second stage's iterative correction). To the best of our knowledge, this is the first study to introduce deep unrolling optimization into safe RL.

\begin{figure}
      \centering
        \includegraphics[width=1\linewidth]{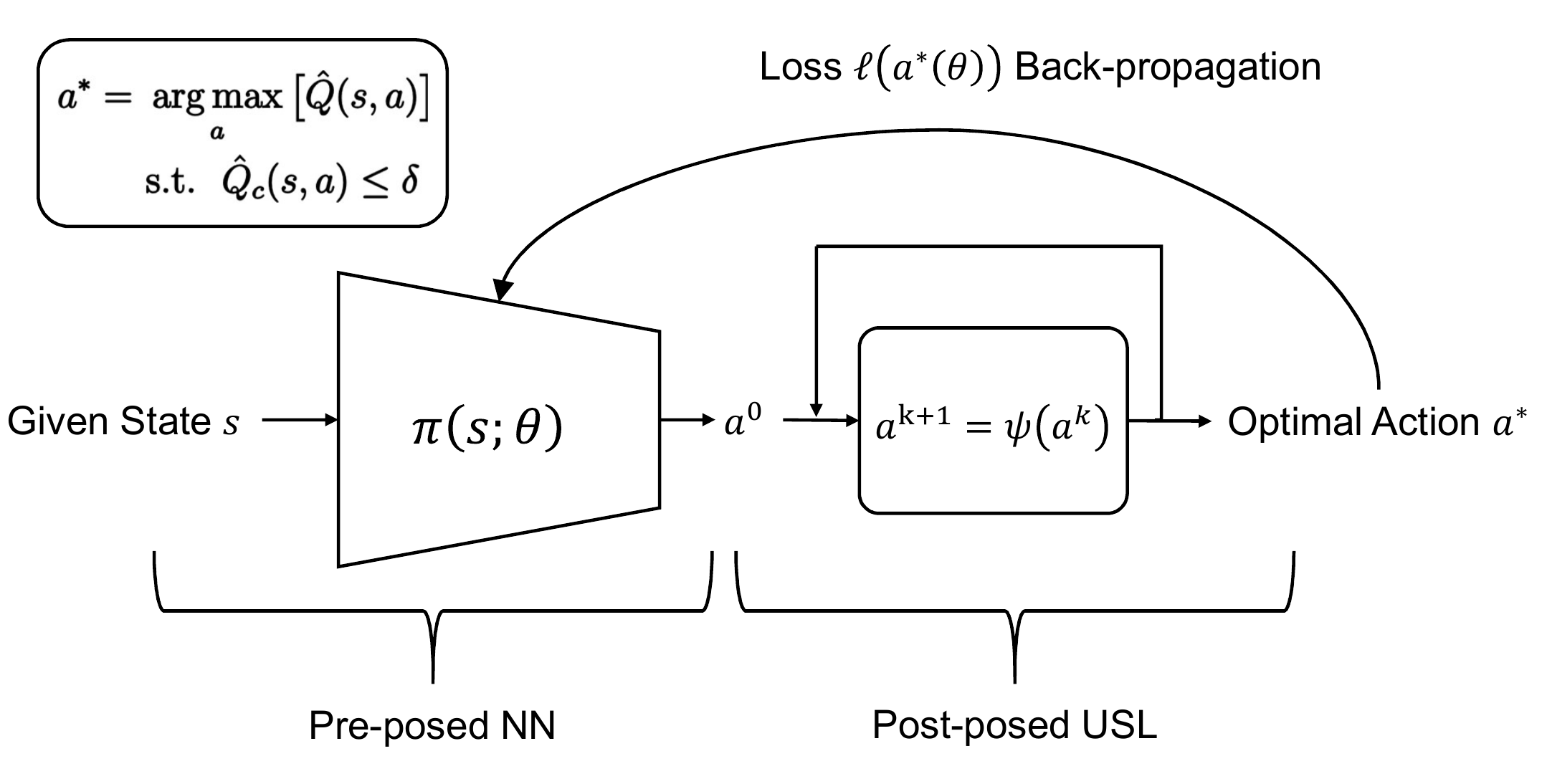}
      \caption{The deep unrolling architecture for safe RL. At each decision-making step, the pre-posed policy network outputs near-optimal action $a_0$; the post-posed unrolling safety layer (USL) takes $a_0$ as the initial solution and iteratively performs gradient-based corrections to enforce the hard constraint. Back-propagation of the state-wise constrained objective guides the policy optimization.}
      \label{fig:usl}
\end{figure}

\subsection{Stage 1: Policy Network as Approximate Solver}
We train a parametric neural network in the first stage as the approximate solver to problem~\eqref{infer}, which aims to output sub-optimal actions via naive forward computing. Different from Donti et al.~\shortcite{donti2021dc3} that simply applies $\ell_2$ regular term to the objective function, we use the exact penalty function~\cite{ijcai2022-520} as the alternative. The merit is that one can construct an equivalent counterpart for problem~\eqref{infer} with a finite penalty factor $\kappa$ as
\begin{equation}\label{obj}
\small
    \ell(\pi) = \mathbb{E}_{\mathcal{D}}\big[ -Q^\pi(s,\pi(s)) + \kappa\cdot\max\{0, Q^\pi_c(s,\pi(s)) - \delta\} \big].
\end{equation}

\begin{proposition}\label{exact}
Let $\mathcal{L}(\pi,\lambda)$ denote the Lagrangian function $ \mathbb{E}_{\mathcal{D}} \big[-Q^\pi(s,\pi(s)) + \lambda(s) \big(Q^\pi_{c}(s,\pi(s)) - \delta \big)\big]$. 
Assume that the optimal $\pi^*$ and $\lambda^*$ exist for the dual problem $\max_{\lambda\geq0} \min_\pi \mathcal{L}(\pi,\lambda)$. If $\kappa \geq ||\lambda(s)||_\infty$, it holds that \[\min \ell(\pi) \iff \max_{\lambda\geq0} \min_\pi \mathcal{L}(\pi,\lambda)\]
\end{proposition}
\begin{proof}
The exactness of the penalty function can be referred to  our previous work~\cite{ijcai2022-520}.
\end{proof}

We use a fixed $\kappa$ as a hyper-parameter in the practical implementation and find that a large constant ($\kappa = 5$ in our experiments) is  empirically effective across different tasks even if the supremum of Lagrange multipliers is intractable to estimate. Moreover, if the actual $\lambda(s)$ tends to be positively infinitive for some critically dangerous states, there would be numerical issues for optimization-based approaches. On the contrary, the objective function~\eqref{obj} under those circumstances can be regarded as a penalty method by adding regularization terms and only gives a sub-optimal initial solution for the next stage. Fortunately, the proposed two-stage architecture does not require an optimal solution in the first stage, and the joint training and inference process with post-projection can tackle this problem to some extent.

 \begin{figure*}
      \centering
        \includegraphics[width=1.0\linewidth]{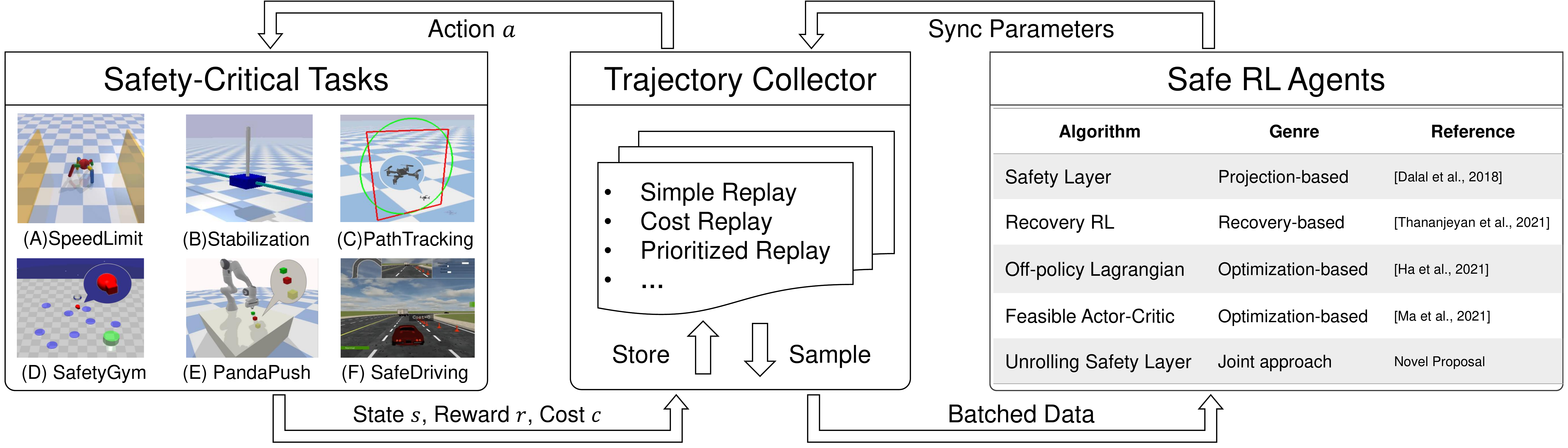}
      \caption{The schema of SafeRL-Kit. All the algorithms are implemented under off-policy settings and evaluated on six safety-critical tasks with the universal binary cost indicator and a constant constraint threshold.}
      \label{fig:framework}
\end{figure*}

\subsection{Stage 2: Gradient-based Projection}
The approximate solver in stage 1 may still output infeasible actions for the following reasons: (a) The supremum of Lagrangian multipliers in Proposition~\ref{exact} is hard to obtain, and we only settle $\kappa$ as a fixed, large but sub-optimal hyper-parameter. (b) The inherent issues of safe RL, such as the approximation in the modeling, sample-based learning, distributional shift, etc., make it possible for the end-to-end actor to violate hard constraints in the policy execution.

To address the above issue, the post-posed projection performs gradient steps to rectify the hard constraint from the initial iteration point $a^0 \sim \pi(\cdot|s)$. $\psi(\cdot)$ is the differentiable operator that takes intermediate action $a^k$ at $k_\mathrm{th}$ iteration as input and performs a gradient descent towards the constraint-violation wrapped with ReLU function:
\begin{equation}
    \psi(a^k) = a^k - \frac{\eta}{\mathcal{Z}_k} \cdot \frac{\partial}{\partial{a^k}}\big[Q_c(s,a_k)- \delta \big]^+.
\end{equation}

Here $\mathcal{Z}_k = || \frac{\partial}{\partial{a^k}}\big[Q_c(s,a_k)- \delta \big]^+||_\infty$ is the normalization factor that rescales the gradients on $a_k$, and therefore the hyper-parameter $\eta$ determines the maximum step size of the action change.

Notably, the iterative executions of  $\psi(\cdot)$ do not always converge to global (or even local) optima for the primal constrained optimization problem~\eqref{infer}. Nevertheless, such a method is highly effective in practice if the initial iteration point is close to the optimal solution~\citep{panageas2019first,donti2021dc3}. This fact emphasizes the necessity for training a pre-posed policy network via the exact penalty regularization, which provides a non-pathological initialization for USL. By means of minimizing the objective function~\eqref{obj}, the output of $\pi(\cdot|s)$ may be still infeasible sometimes but already close to the optimal action $a^*$. Thus, the sequence of $a^{k+1} = \psi(a^k)$ is expected to converge to $a^*$ when $k\rightarrow +\infty$. However, letting $k\rightarrow +\infty$ is not practical in use, and thus we set an upper limit $K$ as the maximum iterations of USL. Note that the value of $K$ needs to match the gradient step-size factor $\eta$. For example, we set $\eta = 0.05$ and $K = 20$ to enable USL to degrade the normalized action from $1$ to $0$ within maximum iterations.

More algorithmic details are summarized in Appendix A. 

\section{SafeRL-Kit: A Systematic Implementation}

To facilitate further research in this area, we release SafeRL-Kit\footnote{Project page: \url{https://sites.google.com/view/saferlkit}}, a reproducible and open-source safe RL toolkit as shown in Figure~\ref{fig:framework}. In brief, SafeRL-Kit contains a list of representative algorithms that address safe learning from different perspectives. Potential users can also incorporate domain-specific knowledge into appropriate baselines to build more competent algorithms for their tasks of interest. Furthermore, SafeRL-Kit is implemented in an off-policy training pipeline, which provides unified and efficient interfaces for fair comparisons among different algorithms on different benchmarks.

\subsection{Safety-critical Benchmarks}
SafeRL-Kit includes six safety-critical benchmarks, ranging from basic robotic control to autonomous driving,  which are well-explored in recent literature~\cite{yuan2021safe,ray2019benchmarking,li2021metadrive}. A short description of the benchmarks is presented below:

\begin{description}[leftmargin=0mm]
    \item[(A) SpeedLimit.] The four-legged ant runs along the avenue and receives a cost signal when exceeding the velocity limit.
    \item[(B) Stabilization.] The cart pole is rewarded for keeping itself upright while being constrained by angular velocity.
    \item[(C) PathTracking.]  The quadrotor tracks the green circular trajectory and receives a cost signal if it leaves the area allowed to fly bounded within the red rectangular.
    \item[(D) SafetyGym-PG.] The mass point moves to the green goal and is required to get rid of blue hazards.
    \item[(E) PandaPush.] The robotic arm pushes the green cube to the destination while avoiding collisions with the red cube.
    \item[(F) SafeDriving.] The autonomous vehicle learns to reach the navigation land markers as quickly as possible but is not allowed to collide with other vehicles or be out of the road.
\end{description}

\begin{figure*}
      \centering
        \includegraphics[width=1\linewidth]{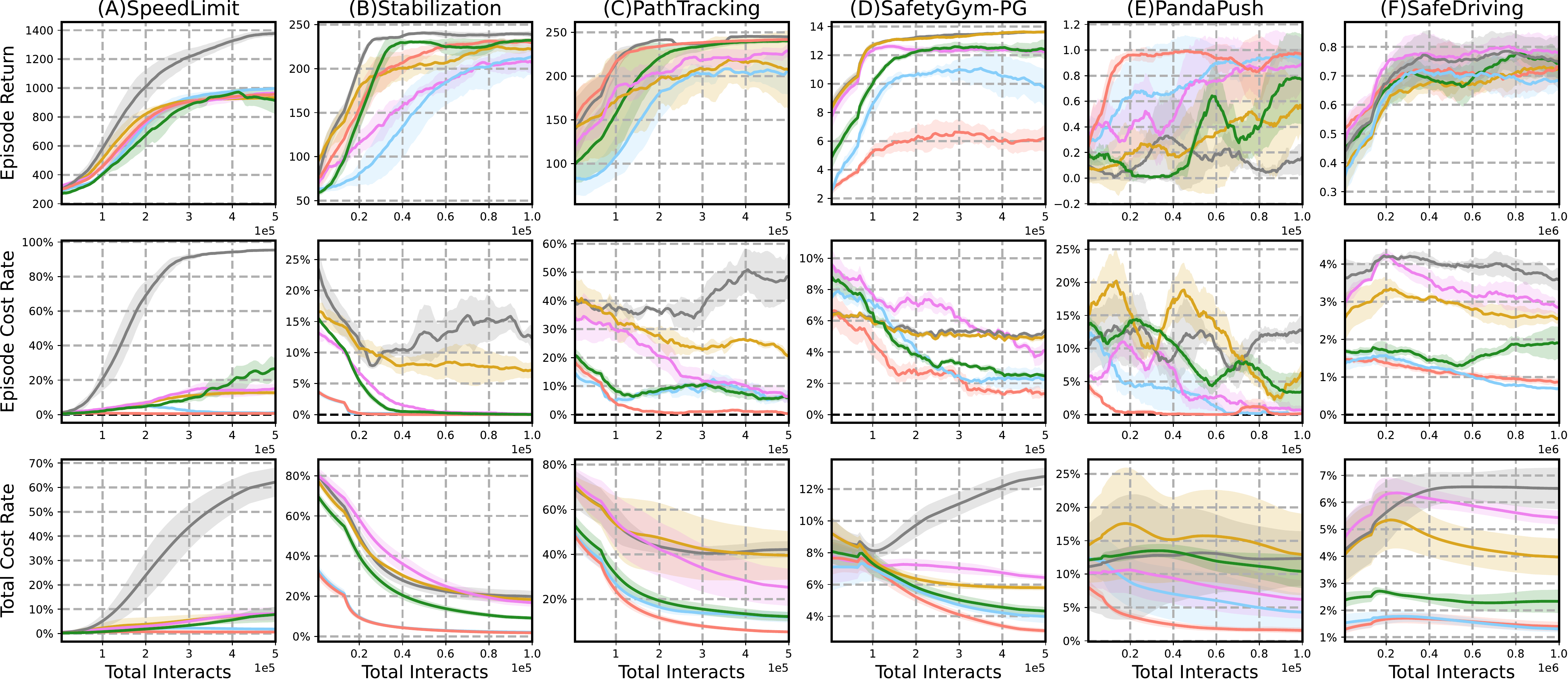}
        \includegraphics[width=0.75\linewidth]{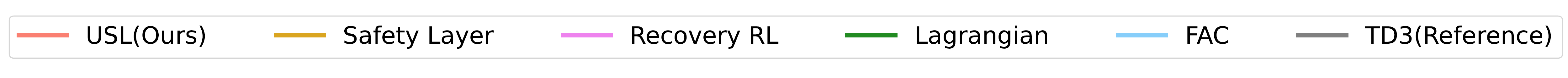}
      \caption{Learning curves of different algorithms on safety-critical tasks. The x-axis is the number of interactions. The y-axis represents episodic return (top line), episodic cost rate (middle line) and total cost rate (bottom line), respectively. The dashed line is the expected zero-cost limit in the test time.}
      \label{fig:main}
\end{figure*}

More detailed environment descriptions can be referred to Appendix B.1.

\subsection{Safe Learning Algorithms}

SafeRL-Kit includes five safe learning methods addressing safety-critical tasks from different perspectives. 
\begin{description}[leftmargin=0mm]
    \item[Safety Layer]\cite{dalal2018safe} for safety projection
    \begin{equation*}
    \small
        a^* = \arg\min_a \frac{1}{2} ||a-\pi_\theta(s) ||^2\quad \mathrm{s.t.} \ \  g_\omega(s)^\top a + \bar{c}(s) \leq \delta.
    \end{equation*}
    \item[Recovery RL]\citep{thananjeyan2021recovery} for safety recovery 
    
    $
\qquad\quad  a_t = 
        \begin{cases}
            \pi_\text{task}(s_t),& \text{if } Q^\pi_\text{risk}\big(s_t,\pi_\text{task}(s_t)\big)\leq \delta\\
            \pi_\text{risk}(s_t),  & \text{otherwise}
        \end{cases}.
$
    \item[Off-Policy Lagrangian Method]\citep{ha2020learning} for safety optimization
     \begin{equation*}
    \small
    \mathop{\max}_{\lambda \geq 0} \mathop{\min}_{\theta} \mathbb{E}_{\mathcal{D}} -Q^\pi(s,\pi_\theta(s)) + \lambda \big(Q^\pi_{c}(s,\pi_\theta(s)) - \epsilon \big).
\end{equation*}
    \item [Feasible Actor Critic (FAC)]\citep{ma2021feasible}  for safety optimization
\begin{equation*}
\small
    \mathop{\max}_{\xi} \mathop{\min}_{\theta} \mathbb{E}_{\mathcal{D}} -Q^\pi(s,\pi_\theta(s)) + \lambda_\xi(s) \big(Q^\pi_{c}(s,\pi_\theta(s)) - \epsilon \big).
\end{equation*}
\item[Unrolling Safety Layer (USL).] A joint approach proposed in this paper combining safety projection and optimization.
\end{description}

All the above algorithms in SafeRL-Kit are implemented under the off-policy Actor-Critic architecture. Although these model-free algorithms may inevitably encounter cost signals in the training process, they still enjoy better sample efficiency with fewer unsafe transitions compared with on-policy implementations~\cite{ray2019benchmarking}, and can better leverage human demonstration if needed. The essential updates of backbone networks uniformly follow TD3~\citep{fujimoto2018addressing}, and thus we can perform a fair evaluation to see which of them are best suited for safety-critical tasks.

More details of the implemented algorithms can be referred to the Appendix B.2. 

\subsection{Cost Function and Evaluation Metrics}
Without loss of generality, we uniformly designate the cost function as a binary indicator (1 for unsafe transitions; 0 for other cases), and our experiments aim to obtain stationary policies that adhere to zero cost signals. It is worth noting that some works define cost more prospectively, for example, using the distance from the car to the road boundary in autonomous driving scenarios~\cite{chen2021safe}. We omit any task-dependent hand-crafting in our comparative study since they overly rely on domain-specific knowledge and are sometimes intractable on complex tasks. Instead, receiving an instantaneous cost signal is much more straightforward and can be generalized to related tasks. 

Considering the properties of safety-critical tasks and the definition of the cost function, we employ the following metrics for the joint evaluation:
\begin{description}[leftmargin=0mm]
    \item[Episodic Return] $\triangleq \text{sum of rewards in the test time}.$  It indicates how well the agent finishes the original task.
    \item[Episodic Cost Rate] $\triangleq\frac{\text{number of cost signals}}{\text{length of the episode}}.$ It indicates how safe the agent is in the test time.
    \item[Total Cost Rate] $\triangleq\frac{\text{total number of cost signals}}{\text{total number of training steps}}.$ It indicates how safe the agent is in the whole training process.
\end{description}

\begin{table*}[t]
\begin{center}
\begin{small}
\begin{sc}
\caption{Mean performance at convergence  with 95\% confidence interval for different algorithms on safety-critical tasks.}
\label{tab:perf}
\setlength\tabcolsep{3pt}
\resizebox{1\textwidth}{42mm}{
\begin{tabular}{|ll|cccccc|}
\toprule
 \multicolumn{2}{|c|}{\small{Tasks}} & \small{USL(ours)} & \small{Safety Layer} & \small{Recovery RL}  & \small{Lagrangian} & \small{FAC} & \small{TD3(ref)}\\
\midrule
\multirow{3}*{\scriptsize{(A)Speedlimit}} & \scriptsize Ep Return & $965.97\pm 5.09$& $935.69\pm9.11$ & $965.73\pm 3.52$& $897.32\pm 77.80$& \bm{$978.77\pm 12.12$}& $1382.16\pm 19.65$\\
~ & \scriptsize Ep CostRate(\%) & \bm{$0.63\pm0.55$} & $12.73\pm2.71$ &$15.10\pm4.07$ & $27,17\pm14.00$& $1.07\pm0.58$& $95.34\pm1.03$\\
~ & \scriptsize Tot CostRate(\%) & \bm{$0.63\pm0.06$} & $8.00\pm0.75$ &$8.25\pm2.74$ & $8.01\pm3.28$& $1.85\pm0.40$&$62.56\pm5.89$\\
\specialrule{0em}{1pt}{1pt}
\hline
\specialrule{0em}{1pt}{1pt}
\multirow{3}*{\scriptsize{(B)Stabilization}} & \scriptsize Ep Return & $228.18\pm 3.88$& $222.80\pm 15.22$& $204.23\pm3.92$ & \bm{$231.61\pm 2.12$}& $214.20\pm 19.38$& $238.47\pm 3.99$\\
~ & \scriptsize Ep CostRate(\%) & \bm{$0.00\pm0.00$} & $6.98\pm2.96$ &$0.10\pm0.04$ & $0.03\pm0.06$& $0.02\pm0.03$& $12.89\pm4.08$\\
~ & \scriptsize Tot CostRate(\%) & \bm{$1.77\pm0.23$} & $18.27\pm2.06$ &$16.63\pm1.98$ & $9.05\pm0.85$& $2.07\pm0.19$&$20.01\pm2.87$\\
\specialrule{0em}{1pt}{1pt}
\hline
\specialrule{0em}{1pt}{1pt}
\multirow{3}*{\scriptsize{(C)Pathtracking}} & \scriptsize Ep Return & \bm{$241.74\pm 0.88$}& $205.93\pm 46.73$& $229.05\pm7.16$ & $240.50\pm 2.03$& $218.28\pm 28.99$& $248.80\pm 0.54$\\
~ & \scriptsize Ep CostRate(\%) & \bm{$0.17\pm0.18$} & $18.79\pm5.52$ &$6.22\pm3.91$ & $5.77\pm2.27$& $6.33\pm3.55$& $48.40\pm9.49$\\
~ & \scriptsize Tot CostRate(\%) & \bm{$5.36\pm0.60$} &$39.44\pm10.66$ & $25.08\pm8.27$ & $12.18\pm2.10$& $11.92\pm2.43$&$42.22\pm3.74$\\
\specialrule{0em}{1pt}{1pt}
\hline
\specialrule{0em}{1pt}{1pt}
\multirow{3}*{\scriptsize{(D)Safetygym}} & \scriptsize Ep Return & $6.36\pm 0.90$& \bm{$13.64\pm 0.05$}& $12.24\pm0.26$ & $12.28\pm 0.87$ & $ 9.66\pm 1.24$& $13.65\pm 0.12$\\
~ & \scriptsize Ep CostRate(\%) & \bm{$1.49\pm0.74$} & $5.17\pm0.47$ &$4.06\pm1.68$ & $2.47\pm0.41$& $1.84\pm.88$& $5.41\pm0.16$\\
~ & \scriptsize Tot CostRate(\%) & \bm{$3.07\pm0.15$} & $5.79\pm0.16$ &$6.40\pm0.22$ & $4.31\pm0.27$& $3.97\pm0.42$&$12.80\pm0.55$\\
\specialrule{0em}{1pt}{1pt}

\hline
\specialrule{0em}{1pt}{1pt}
\multirow{3}*{\scriptsize{(E)Pandapush}} & \scriptsize Ep Return & \bm{$0.96\pm 0.04$}& $0.55\pm 0.22$& $0.89\pm0.19$ & $0.77\pm 0.35 $& $0.92\pm 0.08$& $0.16\pm 0.13$\\
~ & \scriptsize Ep CostRate(\%) & \bm{$0.17\pm0.17$} & $6.28\pm4.44$ &$0.66\pm 0.57$ & $6.28\pm4.44$& $3.48\pm2.55$& $12.56\pm4.70$\\
~ & \scriptsize Tot CostRate(\%) & \bm{$1.56\pm0.43$} & $12.89\pm6.11$ &$6.17\pm2.87$ & $10.37\pm2.52$& $4.28\pm2.84$&$12.35\pm4.06$\\
\specialrule{0em}{1pt}{1pt}

\hline
\specialrule{0em}{1pt}{1pt}
\multirow{3}*{\scriptsize{(F)Safedriving}} & \scriptsize Ep Return & $0.73\pm 0.04$& $0.73\pm 0.05$& \bm{$0.78\pm0.06$} & $0.74\pm 0.05$& $0.69\pm 0.04$& $0.80\pm 0.06$\\
~ & \scriptsize Ep CostRate(\%) & $0.85\pm0.14$ & $2.59\pm0.22$ &$2.83\pm0.38$ & $1.85\pm0.98$& \bm{$0.66\pm0.10$}& $3.81\pm0.51$\\
~ & \scriptsize Tot CostRate(\%) & \bm{$1.30\pm0.17$} & $3.96\pm0.68$ &$5.42\pm0.25$ & $2.33\pm0.44$& $1.30\pm0.13$ & $6.22\pm0.82$\\
\specialrule{0em}{1pt}{1pt}
\bottomrule
\end{tabular}
}
\end{sc}
\end{small}
\end{center}
\end{table*}

\begin{table*}[t]
\begin{center}
\begin{small}
\begin{sc}
\caption{Ablation study for USL on two representative tasks.}
\label{tab:ab}
\resizebox{1\textwidth}{15mm}{
\begin{tabular}{|c|ccc|ccc|}
\toprule
 Tasks & \multicolumn{3}{c|}{Speedlimit}  &\multicolumn{3}{c|}{Pathtracking}\\
\midrule
USL Models & \scriptsize{Ep Return} & \scriptsize{Ep CostRate(\%)} & \scriptsize{Tot CostRate(\%)} & \scriptsize{Ep Return} & \scriptsize{Ep CostRate(\%)} & \scriptsize{Tot CostRate(\%)} \\
\midrule
\scriptsize{Stage 1 + Stage 2} & $965.97\pm 5.09$ & $0.63\pm0.55$& $0.63\pm0.06$ & $241.74\pm 0.80$ & $0.09\pm0.09$& $5.35\pm0.60$\\
\scriptsize{Stage 1 only}& $1016.03\pm 29.17$  & $5.05\pm2.69$& $2.53\pm 0.39$& $242.32\pm 0.27$ & $0.49\pm0.09$& $8.74\pm0.72$ \\
\scriptsize{Stage 2 only}& $989.12\pm 96.62$ & $38.87\pm15.50$& $13.25\pm7.82$& $211.88\pm18.28$ & $0.62\pm0.15$& $18.24\pm1.66$ \\
\scriptsize{Unconstrained} & $1382.16\pm 19.65$ & $95.34\pm1.03$& $62.56\pm5.89$& $244.80\pm0.54s$ & $24.20\pm4.75$& $42.22\pm3.74s$ \\
\bottomrule
\end{tabular}
}
\end{sc}
\end{small}
\end{center}
\end{table*}

\begin{table*}[t]
\begin{center}
\begin{small}
\begin{sc}
\caption{Computational efficiency for different algorithms ($K=20, \eta=0.05$ for USL).}
\label{tab:perf2}
\setlength\tabcolsep{12pt}

\resizebox{\textwidth}{10mm}{
\begin{tabular}{|c|cccccc|}
\toprule
Metrics & USL(ours) & Recovery &  Safety Layer & Lagrangian & FAC & TD3(ref) \\
\midrule
\scriptsize Normalized Inference Time& 5.0& 1.2& 1.6& 1.0& 1.0& 1.0\\

\scriptsize Average Inference Time (s) &
25E-4& 6E-4& 8E-4& 5E-4& 5E-4& 5E-4\\

\scriptsize Max Control Frequency (Hz) &
400& 1666& 1250& 2000& 2000& 2000\\

\bottomrule
\end{tabular}
}
\end{sc}
\end{small}
\end{center}
\end{table*}

\section{Experiments}

In this section, we empirically evaluate model-free RL toward safety-critical tasks based on SafeRL-Kit and investigate the following research questions.

\subsection{Q1: How applicable and robust are the algorithms regarding state-wise constraints?}

We plot the learning curves of each algorithm on different tasks over five random seeds in Figure~\ref{fig:main} and report their mean performance at convergence in Table~\ref{tab:perf}. Detailed hyper-parameters settings are presented in the supplementary material C.1.

TD3~\cite{fujimoto2018addressing} is used as the unconstrained reference for the upper bounds of reward performance and cost signals. On the \emph{SpeedLimit} task, the safety constraint  ($velocity < 1.5 m/s$) is easily violated since the ant is able to move much faster for higher rewards. Thus, we can clearly observe that the TD3 agent achieves over $95\%$ episodic cost rate at convergence while other risk-aware algorithms in SafeRL-kit suppress the explosion of cost rate. On other tasks, the cost signals are more sparse, such as the accidental collisions with obstacles in autonomous driving. Nevertheless, the evaluated algorithms still effectively reduce the likelihood of safety violations.

The empirical results reveal that Safety Layer and Recovery RL are comparatively ineffective in reducing the cost return. For Safety Layer, the main reasons are that the linear approximation to the cost function brings about non-negligible errors, and the single-step correction is myopic for future risks. For Recovery RL, the estimation error of $Q_\text{risk}$ is the major factor affecting its efficacy. 

By contrast, Off-policy Lagrangian and FAC have significantly lower cumulative costs. However, Lagrangian-based methods may suffer from the inherent issues due to primal-dual ascents. For one thing, the Lagrangian multiplier tuning causes oscillations in learning curves. For another thing, the performance may heavily depend on the Lagrangian multipliers' initialization and learning rate. According to the sensitivity analysis, we find that Lagrangian-based methods are susceptible to the learning rate of the Lagrangian multiplier(s) in stochastic primal-dual optimization. First, the oscillating $\lambda$ causes non-negligible deviations in the learning curves. Second, the increasing $\eta_\lambda$ may degrade the performance dramatically. The phenomenon is especially pronounced in FAC, which has a multiplier network to predict the state-dependent $\lambda(s;\xi)$. Consequently, we suggest to ensure $\eta_\lambda \ll \eta_\theta$ in practice. We put the sensitivity analysis of these two algorithms in Appendix C.2.

 \begin{figure}
      \centering
        \includegraphics[width=1\linewidth]{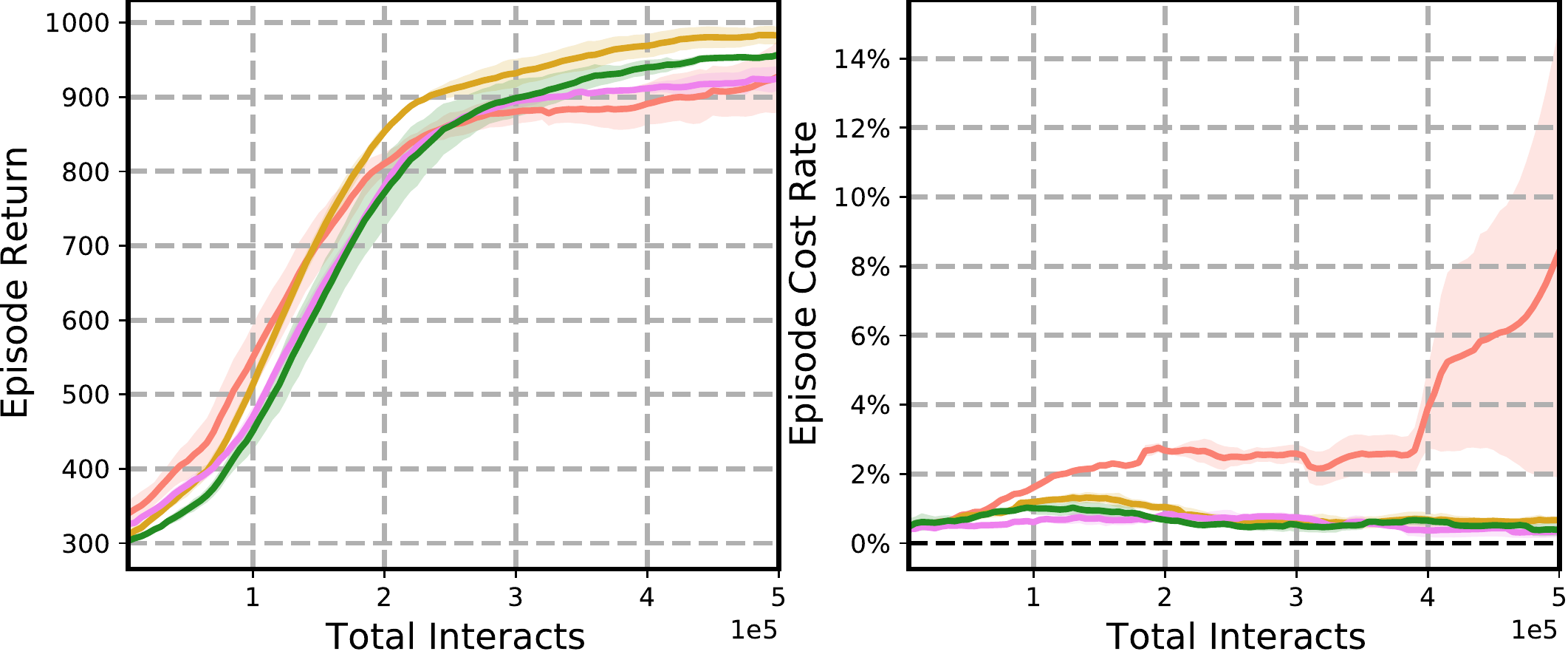}
        \includegraphics[width=0.75\linewidth]{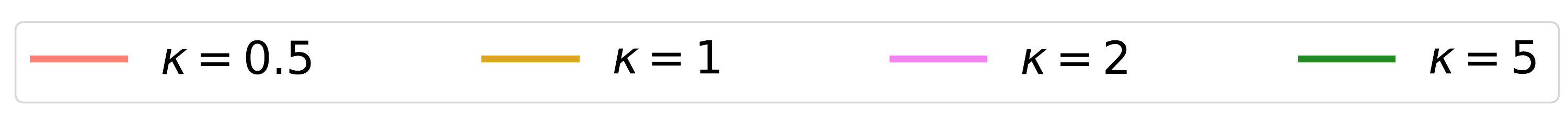}
      \caption{Sensitivity analysis on penalty factor $\kappa$.}
      \label{fig:kappa}
\end{figure}

In summary, the proposed USL is clearly effective for learning constraint-satisfying policies. First, USL achieves higher or competitive returns while adhering to (almost) zero cost return across different tasks. Second, USL features minor standard deviations and oscillations, which demonstrates its robustness. At last, USL generally converges with fewer interactions, which is crucial in sample-expensive risky environments.
The underlying reason is that the optimization stage is equivalent to FAC but the state-dependent Lagrangian multipliers are reduced to a single fixed hyper-parameter. Meanwhile, the consistent loss function stabilizes the training process compared with primal-dual optimization. Furthermore, the projection stage explicitly enforces the state-wise constraints intractable in naive forward computing. 

Additional experiments for comparing the state-wise constraint, the episodic constraint and reward shaping are placed in the supplementary material C.3.

\subsection{Q2: How to account for the importance of the two stages in USL?}

To better understand the importance of the two stages in our approach, we perform an ablation study as shown in Table~\ref{tab:ab} and confirm that the two stages must work jointly to achieve the desired performance.  An intuitive example is that the solution derived by standard RL may be far away from the desired optimal safe action on tasks such as \emph{SpeedLimit}. Thus, directly post-optimizing over $Q^\pi_c$ may not necessarily converge to $a^*$, and the agent still has a 38\% cost rate. Instead, if the initial solution from Stage 1 is close to $a^*$, it can serve as a valid candidate and the cost rate goes down to 0.63\%. Note that, when the unconstrained action is not that far from the safe set, such as on the \emph{PathTracking} task, both the optimization and projection parts can effectively degrade the cost rate from 24\% to less than 1\%. However, using only Stage 2 is inferior on episodic return, which is an inherent flaw of the projection-based method.

\subsection{Q3: How sensitive is USL to its hyper-parameters and how to tune them?}

We study the impacts of two pivotal hyper-parameters in USL, namely the penalty factor $\kappa$ in the training objective and the maximum iterative number $K$ in the post-projection, on the \emph{SpeedLimit} task. For $\kappa$, Figure~\ref{fig:kappa} shows that final policies are insensitive to its value, and the learning curves are almost identical for sufficiently large $\kappa$ values. By contrast, a small $\kappa$ value may degenerate the first stage of USL into a ``soft'' regularization method. In our experiments, we find $\kappa = 5$ generally achieves good performance across different tasks. For $K$, Figure~\ref{fig:K} shows that USL can enforce the hard constraint within five iterations at most decision-making steps, indicating the possibility of navigating the trade-off between constraint satisfaction and computational efficiency. We set $K = 0$ in the sensitivity study and show that the single optimization in Stage 1 can not lead to zero cost return without the aid of the post-posed projection.

\begin{figure}
      \centering
        \includegraphics[width=1\linewidth]{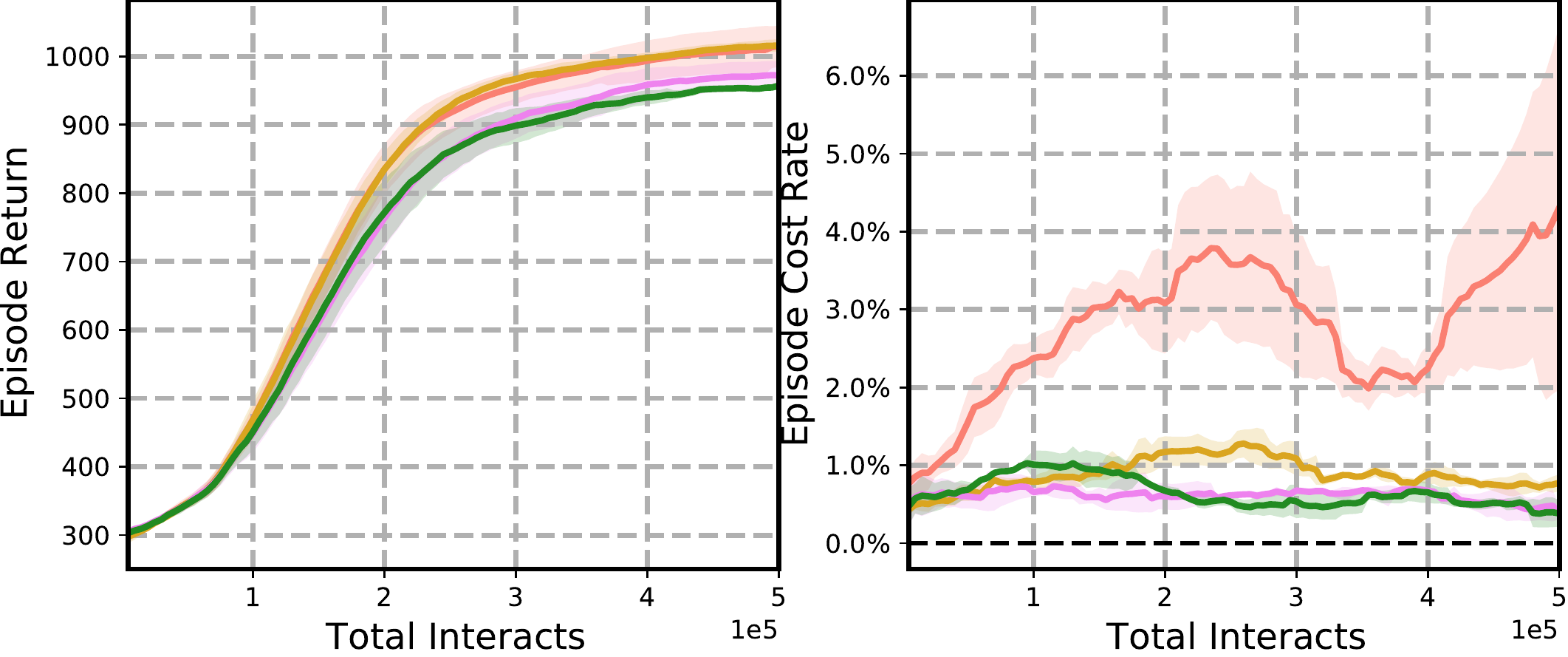}
        \includegraphics[width=0.75\linewidth]{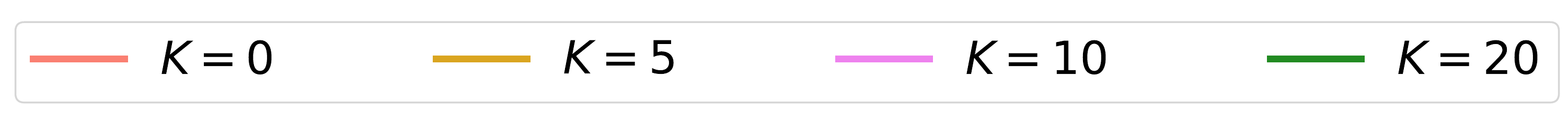}
      \caption{Sensitivity analysis on iterative number $K$.}
      \label{fig:K}
\end{figure}
\subsection{Q4: How is the computational efficiency of USL with the additional iterative steps?}
The two-stage architecture of USL inevitably brings concerns on computational feasibility in real-world applications. We compare different algorithms on a mainstream computing platform (Intel Core i7-9700K, NVIDIA GeForce RTX 2070). Table~\ref{tab:perf2} shows that USL takes around 4-5 times the inference time of the unconstrained TD3 but still achieves an admissible 400 Hz control frequency. Having said that, we will leave the efforts on improving the time efficiency of USL to accelerate inference speed as future work.

\section{Conclusions}
In this paper, we perform a comparative study on model-free reinforcement learning toward safety-critical tasks following state-wise safety constraints. We revisit and evaluate related algorithms from the perspective of safety projection, recovery, and optimization, respectively. Furthermore, we propose Unrolling Safety Layer (USL) and demonstrate its efficacy in improving the episodic return and enhancing the safety-constraint satisfaction with an admissible computational complexity. We also present the open-sourced SafeRL-Kit and invite researchers and practitioners to incorporate domain-specific knowledge into the baselines to build more competent algorithms for their tasks.

\section{Acknowledgments}
This work is supported by the National Key R\&D Program of China (2022YFB4701400/4701402), the National Natural Science Foundation of China (No. U21B6002, U1813216, 52265002), and the Science and Technology Innovation 2030 – “Brain Science and Brain-like Research” key Project (No. 2021ZD0201405).

\bibliography{aaai23}
\clearpage
\onecolumn
\setcounter{equation}{0}
\renewcommand{\theequation}{A.\arabic{equation}}

\section{Supplementary Material A: Algorithmic Details}
\begin{algorithm}[H]
\caption{Deterministic Policy Gradients with USL.}
\begin{algorithmic}[1]
\label{algo1}
\REQUIRE deterministic policy network $\pi(s;\theta)$; critic networks $\hat{Q}(s,a;\phi)$ and $\hat{Q}_c(s,a;\varphi)$
\FOR{t \textbf{in} $1,2,...$}
\STATE $a^0_t = \pi(s_t;\theta) + \epsilon,\ \ \epsilon\sim\mathcal{N}(0,\sigma)$.
\FOR{k \textbf{in} 1,2,...$K$ }
\STATE $a^{k}_t = \psi(a^{k-1}_t)$.
\ENDFOR
\STATE Apply $a_t = a^{K}_t$ to the environment.
\STATE Store the transition $(s_t,a_t,s_{t+1},r_t,c_t,d_t)$ in $\mathcal{B}$.
\STATE Sample a mini-batch of $N$ transitions from $\mathcal{B}$.
\STATE $\varphi \leftarrow {\arg\min}_\varphi \mathop{\mathbb{E}}_{\mathcal{B}} \big[\hat Q_c(s,a;\varphi)-\big(c+\gamma_c(1-d) \hat Q_C(s',\pi(s';\theta);\varphi) \big)\big]^2$.
\STATE $\phi \leftarrow {\arg\min}_\phi \mathop{\mathbb{E}}_{\mathcal{B}} \big[\hat Q(s,a;\phi)-\big(r+\gamma(1-d) \hat Q(s',\pi(s';\theta);\phi)\big)\big]^2$.
\STATE $\theta \leftarrow {\arg\min}_\theta \mathop{\mathbb{E}}_{\mathcal{B}}\big[ -\hat Q (s,\pi(s;\theta);\phi) + \kappa\cdot\max\{0, \hat Q_c(s,\pi(s;\theta);\varphi) - \delta\} \big]$.
\ENDFOR
\end{algorithmic}
\end{algorithm}

\begin{algorithm}[H]
\caption{Stochastic Policy Gradients with USL.}
\begin{algorithmic}[1]
\label{algo2}
\REQUIRE stochastic policy network $\pi(\cdot|s;\theta)$; critic networks $\hat{Q}(s,a;\phi)$ and $\hat{Q}_c(s,a;\varphi)$
\FOR{t \textbf{in} $1,2,...$}
\STATE $a^0_t \sim \pi(\cdot|s_t;\theta)$.
\FOR{k \textbf{in} 1,2,...$K$ }
\STATE $a^{k}_t = \psi(a^{k-1}_t)$.
\ENDFOR
\STATE Apply $a_t = a^{K}_t$ to the environment.
\STATE Store the transition $(s_t,a_t,s_{t+1},r_t,c_t,d_t)$ in $\mathcal{B}$.
\STATE Sample a mini-batch of $N$ transitions from $\mathcal{B}$.
\STATE Sample $a' \sim \pi(\cdot|s';\theta)$.
\STATE $\varphi \leftarrow {\arg\min}_\varphi \mathop{\mathbb{E}}_{\mathcal{B}} \big[\hat Q_c(s,a;\varphi)-(c+\gamma_c(1-d) \hat Q_C(s',a';\varphi')  - \alpha \log \pi(a'|s';\theta)\big)\big]^2$.
\STATE $\phi \leftarrow {\arg\min}_\phi \mathop{\mathbb{E}}_{\mathcal{B}} \big[\hat Q(s,a;\phi)-\big(r+\gamma(1-d) \hat Q(s',\pi(s';\theta');\phi') - \alpha \log \pi(a'|s';\theta)\big)\big]^2$.
\STATE Utilize reparameterization to calculate $\tilde{a}(\theta) = \tanh \big( \mu_\theta(s) + \sigma_\theta(s) \odot \xi\big),\ \ \xi \sim \mathcal{N}(0,1)$.
\STATE $\theta \leftarrow {\arg\min}_\theta \mathop{\mathbb{E}}_{\mathcal{B}}\big[ -\hat Q (s,\tilde{a}(\theta);\phi) + \kappa\cdot\max\{0, \hat Q_c(s,\tilde{a}(\theta);\varphi) - \delta\} \big]$.
\ENDFOR
\end{algorithmic}
\end{algorithm}

\section{Supplementary Material B: Implementation Details}
\subsection{B.1 Safety-Critical Benchmarks}
\label{pp4}
\renewcommand{\arraystretch}{1.5}
\begin{longtable}{p{5cm}<{\centering}|p{11.5cm}<{\centering}}

		\hline
		 Environments & Descriptions\\
		 \hline
		 	\begin{minipage}[b]{0.275\columnwidth}
		 \vspace{0.15cm}
		\centering
		(a) SpeedLimit\\
		\raisebox{0\height}{\includegraphics[width=0.8\linewidth,height=2.5cm]{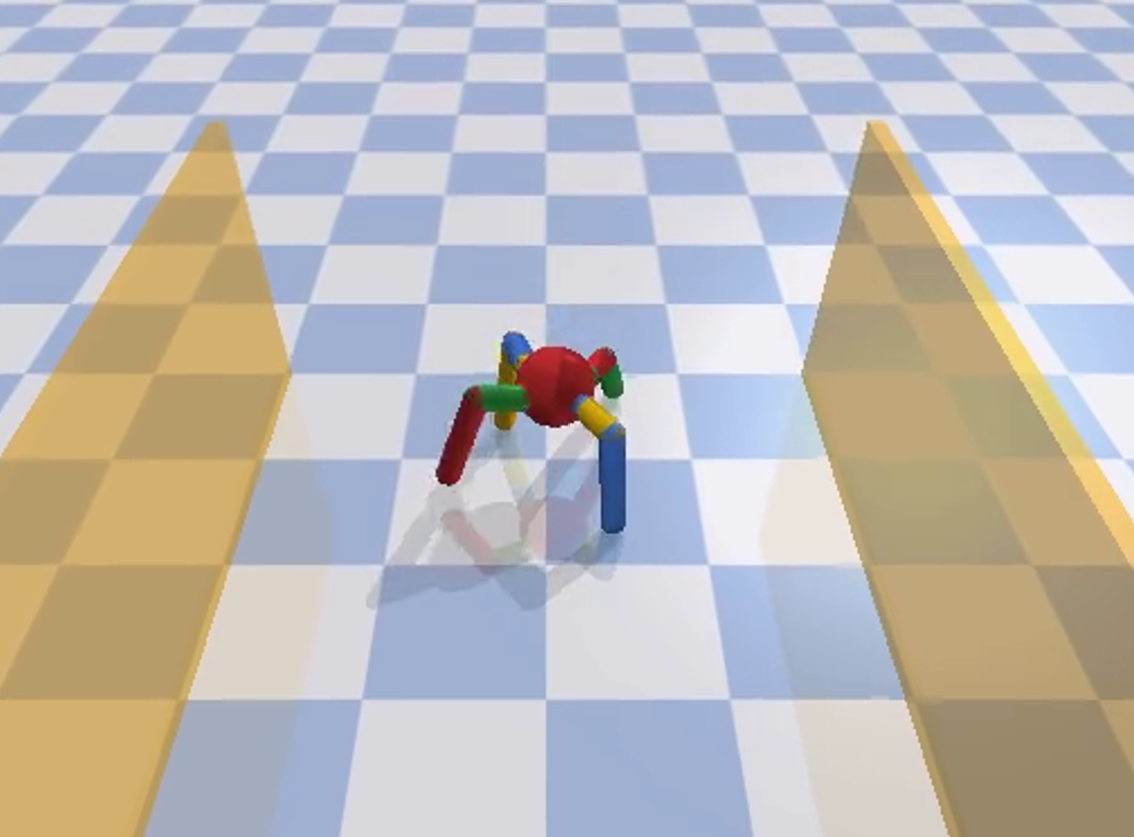}}
		
	\end{minipage} & 		 \begin{minipage}[b]{0.625\columnwidth}
	\vspace{0.15cm}
   Refer to~\citet{zhang2020first}. The observation space ($\mathcal{S} \in \mathbb{R}^33$) contains ego-information, including position, linear velocity, quaternion, the angular velocity, the feet contact forces, etc. The action space ($\mathcal{A} \in \mathbb{R}^8$) denotes the forces applied to each motors. The four-legged ant is rewarded for running along the avenue, which is calculated by the forward distance towards the target $x_\text{target} = +\infty$. However, it is constrained with its linear velocity $\dot x$. In our experiments, the agent receives $+1$ cost signal if $|\dot x| > 0.15 m/s$ or the ant is out of the yellow boundary , i.e., $|y| > 2 m$.
	\end{minipage}\\
		 \hline
		 \begin{minipage}[b]{0.275\columnwidth}
		 \vspace{0.15cm}
		\centering
		(b) Stabilization\\
		\raisebox{0\height}{\includegraphics[width=0.8\linewidth,height=2.5cm]{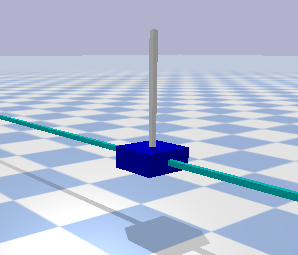}}
		
	\end{minipage} & 		 \begin{minipage}[b]{0.625\columnwidth}
	\vspace{0.15cm}
   Refer to~\citet{yuan2021safe}. The cart-pole is rewarded for keeping itself upright, but is constrained with its angular degree $\theta$ and angular velocity $\dot\theta$. The observation space ($\mathcal{S} \in \mathbb{R}^4$) contains the horizontal position of the cart $x$, the velocity of the cart $\dot x$, the angle of the pole w.r.t vertical $\theta$ and the angular velocity of the pole $\dot\theta$. The action space ($\mathcal{A} \in \mathbb{R}$) is the force $F$ applied to the center of the mass of the cart. The reward function is an instantaneous signal of +1 if the pole is upright ($|\theta| \leq \theta_{\max}$). The agent receives a +1 cost signal if $|\theta| > \theta_{\max}$ or $|\dot\theta| > \dot\theta_{\max}$. In our experiments, we set $\theta_{\max} = \dot\theta_{\max} = 0.2$.
	\end{minipage}\\
		 \hline
		 \begin{minipage}[b]{0.275\columnwidth}
		 \vspace{0.15cm}
		\centering
		(c) PathTracking\\
		\raisebox{0\height}{\includegraphics[width=0.8\linewidth,height=2.5cm]{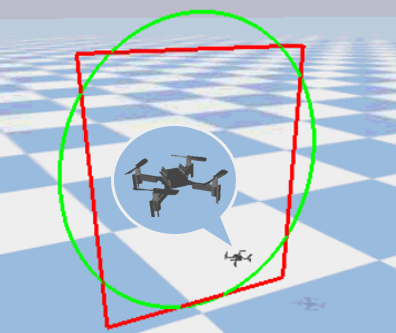}}
		
	\end{minipage} & 		 \begin{minipage}[b]{0.625\columnwidth}
	\vspace{0.15cm}
   Refer to~\citet{yuan2021safe}. The drone (UAV) is rewarded for tracking a circular trajectory, but the safe area is bounded within a smaller rectangular. The observation space ($\mathcal{S} \in \mathbb{R}^6$) contains the translation position $x,z$ and velocity $\dot x, \dot z$ of the drone in the $xz-$plane, as well as the pitch angle $\theta$ and pitch angular velocity $\dot\theta$. The action space ($\mathcal{A} \in \mathbb{R}^2$) denotes the thrusts $[T_1,T_2]$ generated by the two motors. The reward function is in a quadratic form w.r.t to the reference $x$ and $a$. The agent receives $+1$ cost signal if it is out of the area allowed to fly. In our experiments, we set $-0.4m \leq x_\text{safe} \leq 0.4m$ and $0.05m \leq y_\text{safe} \leq 0.9m$.
	\end{minipage}\\
		 \hline
		\begin{minipage}[b]{0.275\columnwidth}
		 \vspace{0.15cm}
		\centering
		(d) Safetygym-PG\\
		\raisebox{0\height}{\includegraphics[width=0.8\linewidth,height=2.5cm]{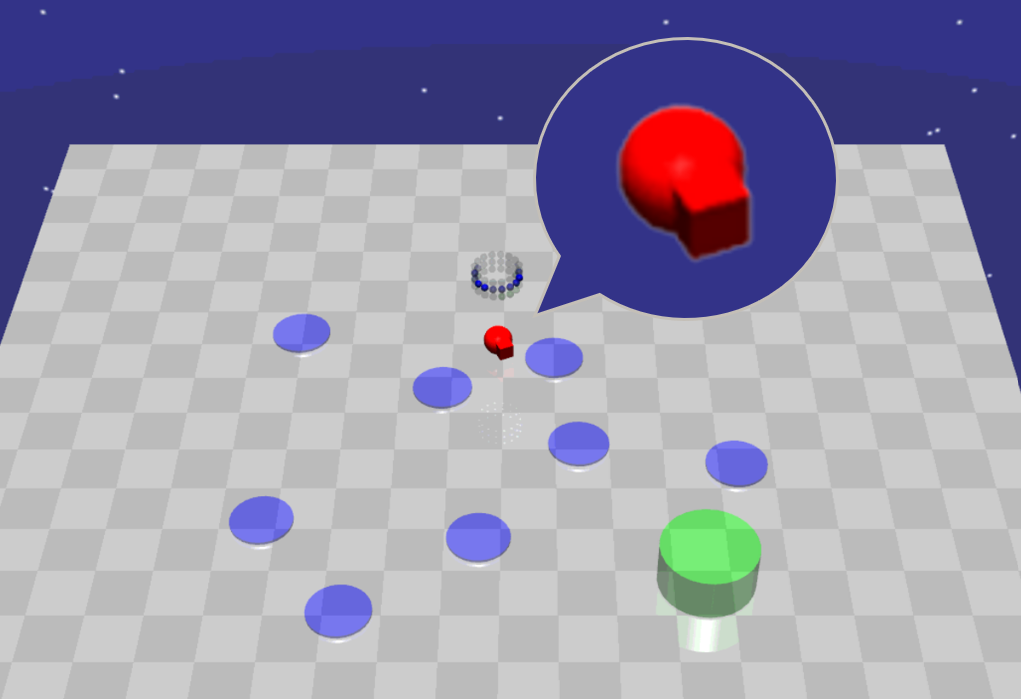}}
		
	\end{minipage} & 		 \begin{minipage}[b]{0.625\columnwidth}
	\vspace{0.15cm}
   Refer to~\citet{ray2019benchmarking}. The observation space ($\mathcal{S} \in \mathbb{R}^60$) contains ego information and Lidar information (towards hazards and the goal respectively), etc. The action space ($\mathcal{A} \in \mathbb{R}^2$) denotes the linear velocity and angular velocity of the point mass. The point mass is rewarded for getting close to the green destination. However, it is the agent receives $+1$ cost signal if it overlaps with the virtual blue hazards. Due to the original point-goal environment is not designed for zero-cost tasks, we inherit the setting as~\citet{zhao2021model}, where the number of hazards is 8 and the radius of them is 0.3m.
	\end{minipage}\\
		 \hline
			\begin{minipage}[b]{0.275\columnwidth}
		 \vspace{0.15cm}
		\centering
		(e) PandaPush\\
		\raisebox{0\height}{\includegraphics[width=0.8\linewidth,height=2.5cm]{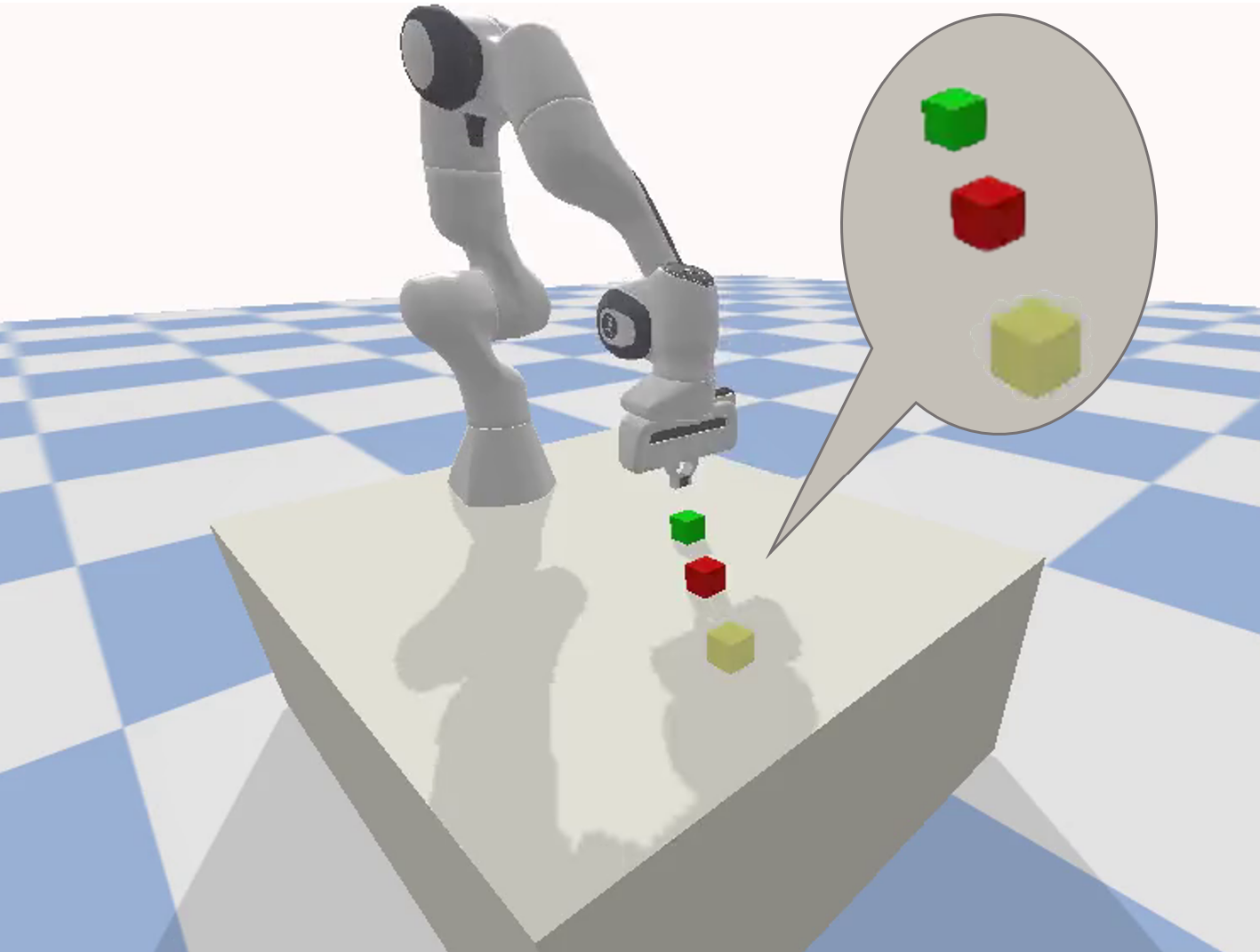}}
		
	\end{minipage} & 		 \begin{minipage}[b]{0.625\columnwidth}
	\vspace{0.15cm}
   Refer to Zhang et al. (2022). The 7-DoF Franka Emika Panda manipulator is required to push the green cube to the destination marked in yellow and the environment returns a sparse reward (0 for finished and -1 for unfinished). We add a red cube in the optimal path and return a +1 if the green cube collides with the obstacle. The observation ($\mathcal{S} \in \mathbb{R}^{21}$) contains ego information, obstacle information and destination information. We adopt position control to move the manipulator end-effector, i.e., the action ($\mathcal{A} \in \mathbb{R}^3$) is the increments on X-Y-Z axis. 
	\end{minipage}\\
		 \hline
				\begin{minipage}[b]{0.275\columnwidth}
		 \vspace{0.15cm}
		\centering
		(f) SafeDrive\\
		\raisebox{0\height}{\includegraphics[width=0.8\linewidth,height=2.5cm]{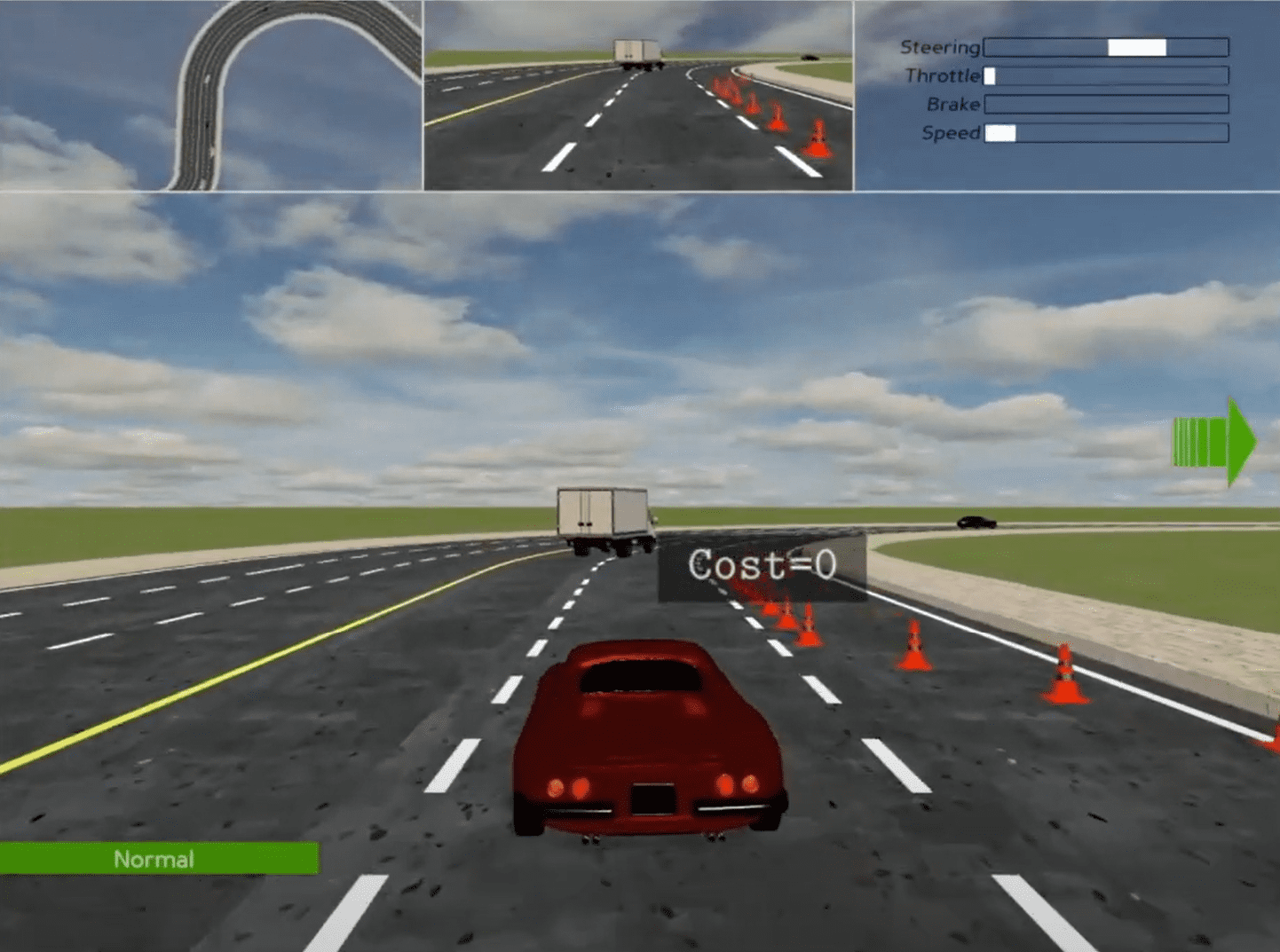}}
		
	\end{minipage} & 		 \begin{minipage}[b]{0.625\columnwidth}
	\vspace{0.15cm}
   Refer to~\citet{li2021metadrive}.  Metadrive is a compositional, lightweight and  realistic platform for vehicle autonomy. Most importantly, it provides pre-defined environments for safe policy learning in autopilots. Concretely, the observation is encoded by a vector containing ego-state, navigation information and surrounding information detected by the Lidar. We control the speed and steering of the car to hit virtual land markers for rewards (By default, \citet{li2021metadrive} conduct a reward shaping by penalizing collisions and overstepping the road; otherwise it would be too hard to learn), and the cost function is +1 if the vehicle collides with other obstacles or it is out of the road.
	\end{minipage}\\
		 \hline
\end{longtable}

\subsection{B.2 Safe Learning Algorithms}

\paragraph{Safety Projection}
This type of method corrects the initial unsafe decision by projecting it back to the safe set. In SafeRL-Kit, the representative model-free method is Safety Layer which is added on the top of the original policy network. Specifically, Safety Layer utilizes a parametric linear model
    $C(s_t,a_t) \approx g(s_t;\omega)^\top a_t + c_{t-1}$
to approximate the single-step cost function with supervised training and solves the quadratic programming as follows
\begin{equation}
\begin{aligned} \label{SafetyLayer}
a_t^* = &\ {\arg\min}_{a}\ \frac{1}{2}|| a - \mu_\theta(s)||^2\\
&\  \mathrm{s.t.} \quad g(s_t;\omega)^\top a_t + c_{t-1} \le \epsilon,
\end{aligned}
\end{equation}
to find the "closest" action to the feasible region. 
Since there is only one compositional cost signal in our problem, the closed-form solution of problem~\eqref{SafetyLayer} is
\begin{equation}
\small
        a_t^* = \mu_\theta(s_t) - \bigg[\frac{g(s_t;\omega)^\top\mu_\theta(s) + c_{t-1} - \epsilon}{g(s_t;\omega)^\top g(s_t;\omega)} \bigg]^+ g(s_t;\omega)
\end{equation}
By the way, the $g_\omega$ is trained from offline data in~\citet{dalal2018safe}. In SafeRL-Kit, we instead learn the linear model with the policy network synchronously, considering the side-effect of distribution shift. We also employ a warm-up for the safety critic in the training process to avoid inaccurate estimation and wrong corrections.

\paragraph{Safety Recovery}
 The critical insight behind safety recovery is to introduce an additional policy that recovers potential unsafe states. In SafeRL-Kit, the representative model-free method is Recovery RL~\cite{thananjeyan2021recovery}. We first learn a safe critic to estimate the future probability of constraint violation as
\begin{equation}\label{Q_risk}
    Q^\pi_\text{risk}(s_t,a_t) = c_t + (1-c_t) \gamma \mathbb{E}_{\pi} Q^\pi_\text{risk}(s_{t+1},a_{t+1}).
\end{equation}
This formulation is slightly different from the standard Bellman equation since it assumes the episode terminates when the agent receives a cost signal. We remove the early-stopping condition for agents to better master complex skills but still preserve the original formulation of $Q^\pi_\text{risk}$ in~\eqref{Q_risk} since it limits the upper bound of the safe critic and eliminates the over-estimation in Q-learning. 
In the phase of policy execution, the recovery policy takes over the control when the predicted value of the safe critic exceeds the given threshold:
\begin{equation}
a_t = 
        \begin{cases}
            \pi_\text{task}(s_t),& \text{if } Q^\pi_\text{risk}\big(s_t,\pi_\text{task}(s_t)\big)\leq \delta\\
            \pi_\text{risk}(s_t),  & \text{otherwise}
        \end{cases}.
\end{equation}
It is of the best practice to store $a_\text{task}$ and $a_\text{risk}$ simultaneously in the replay buffer, and utilize them to train $\pi_\text{task}$ and $\pi_\text{risk}$ respectively in Recovery RL. This technique ensures that $\pi_\text{task}$ can learn from the new MDP, instead of proposing same unsafe actions continuously. Similar to Safety Layer, Recovery RL also has a warm-up stage where $Q^\pi_\text{risk}$ is trained but is not utilized in SafeRL-Kit.

\paragraph{Safety Optimization}
State-wise safe safe RL can be formulated as a constrained sequential optimization problem
\begin{equation}
    a^* = \mathop{\arg\max}_{a} \big[Q^\pi(s,a) \big] \quad \mathrm{s.t.} \ \ Q^\pi_c(s,a) \leq \delta,
    \label{infer2}
\end{equation}
and can be tackled via the dual problem in the  parametric space as follows
\begin{equation}\label{LAG}
\small
    \mathop{\max}_{\lambda \geq 0} \mathop{\min}_{\theta} \mathbb{E}_{\mathcal{D}} -Q^\pi(s,\pi_\theta(s)) + \lambda \big(Q^\pi_{c}(s,\pi_\theta(s)) - \epsilon \big).
\end{equation}
Off-policy Lagrangian applies stochastic primal-dual optimization~\cite{luenberger1984linear} to update primal and dual variables alternatively, which follows as
\begin{equation}\label{pd}
    \small
        \begin{cases}
            \theta \leftarrow \theta + \eta_\theta \nabla_\theta \mathbb{E}_{\mathcal{D}}\big( Q^\pi(s,\pi_\theta(s)) - \lambda Q^\pi_{c}(s,\pi_\theta(s))  \big)\\
            \lambda \leftarrow \big[ \lambda + \eta_\lambda \mathbb{E}_{\mathcal{D}}\big(Q^\pi_{c}(s,\pi_\theta(s)) - \epsilon \big) \big]^+
        \end{cases}
\end{equation}
Notably, the timescale of primal variable updates is required to be faster than the timescale of Lagrange multipliers. Thus, we set $\eta_\theta \gg \eta_\lambda$ in SafeRL-Kit.

The constraint in Off-policy Lagrangian is based on the expectation of the safety critic. Feasible Actor-Critic (FAC)~\cite{ma2021feasible} introduces state-wise constraints for each "feasible" initial states and reformulates~\eqref{LAG} as
\begin{equation}\label{FAC}
\small
    \mathop{\max}_{\lambda \geq 0} \mathop{\min}_{\theta} \mathbb{E}_{\mathcal{D}} -Q^\pi(s,\pi_\theta(s)) + \lambda(s) \big(Q^\pi_{c}(s,\pi_\theta(s)) - \epsilon \big).
\end{equation}
The distinctiveness of problem~\eqref{FAC} is there are infinitely many Lagrangian multipliers that are state-dependent. In SafeRL-Kit, we employ a neural network $\lambda(s;\xi)$ activated by \emph{Softplus} function to map the given state $s$ to its corresponding Lagrangian multiplier $\lambda(s)$. The primal-dual ascents of policy network is similar to~\eqref{pd}; the updates of multiplier network is given by
\begin{equation}
    \xi \leftarrow \xi + \eta_\xi \nabla_\xi  \mathbb{E}_{\mathcal{D}} \lambda(s;\xi) \big(Q^\pi_{c}(s,\pi_\theta(s)) - \epsilon \big).
\end{equation}
Besides, we set a different interval schedule $m_\pi$ (for $\pi_\theta$ delay steps) and $m_\lambda$  (for $\lambda_\xi$ delay steps) in SafeRL-Kit to stabilize the training process inspired by Fujimoto et al.~\shortcite{fujimoto2018addressing}.

\section{Supplementary Material C: Experimental Details}
\subsection{C.1 Hyper-parameter Settings}
\begin{table}[H]
\caption{Hyper-parameters of different safety-aware algorithms in SafeRL-Kit.}
\label{tab:hyper}
\begin{center}
\begin{small}
\begin{sc}
\resizebox{\textwidth}{45mm}{
\begin{tabular}{|l|p{2.25cm}<{\centering}p{2.25cm}<{\centering}p{2.25cm}<{\centering}p{2.25cm}<{\centering}p{2.25cm}<{\centering}|}
\toprule
 Hyper-parameter& Safety Layer &  Recovery RL & Lagrangian & FAC & USL(ours)\\
\midrule
\specialrule{0em}{1pt}{1pt}
Cost Limit $\delta$ & 0.1 & 0.1  & 0.1 & 0.1 & 0.1\\
\specialrule{0em}{1pt}{1pt}
Reward Discount & 0.99 & 0.99  & 0.99 & 0.99 & 0.99\\
\specialrule{0em}{1pt}{1pt}
Cost Discount & 0.99 & 0.99  & 0.99 & 0.99 & 0.99\\
\specialrule{0em}{1pt}{1pt}
Warm-up Ratio & 0.2 & 0.2  & N/A & N/A & N/A\\
\specialrule{0em}{1pt}{1pt}
Batch Size & 256 &  256   &  256  &  256  & 256 \\
\specialrule{0em}{1pt}{1pt}
Critic LR & 3E-4& 3E-4& 3E-4& 3E-4& 3E-4\\
\specialrule{0em}{1pt}{1pt}
Actor LR & 3E-4& 3E-4& 3E-4& 3E-4& 3E-4\\
\specialrule{0em}{1pt}{1pt}
Safe Critic LR & 3E-4& 3E-4& 3E-4& 3E-4& 3E-4\\
\specialrule{0em}{1pt}{1pt}
Safe Actor LR &  N/A& 3E-4&  N/A&  N/A&  N/A\\
\specialrule{0em}{1pt}{1pt}
Multiplier LR &  N/A & N/A & 1E-5& 1E-5&  N/A \\
\specialrule{0em}{1pt}{1pt}
Multiplier Init &  N/A & N/A & 0.0 & N/A &  N/A \\
\specialrule{0em}{1pt}{1pt}
Policy Delay & 2 &  2   &  2  &  2  & 2 \\
\specialrule{0em}{1pt}{1pt}
Multiplier Delay &  N/A &   N/A   &   N/A  &  12  &  N/A\\
\specialrule{0em}{1pt}{1pt}
Penalty Factor $\kappa$ &  N/A &   N/A   &   N/A    &  N/A & 5\\
\specialrule{0em}{1pt}{1pt}
Iterative Step $K$ &  N/A &   N/A   &   N/A    &  N/A & 20\\
\bottomrule
\end{tabular}
}
\end{sc}
\end{small}
\end{center}
\end{table}

\subsection{C.2 Additional Sensitivity Study between Lagrangian, FAC and USL}

In this section, we study the sensitivity to hyper-parameters of Lagrangian-based methods and newly proposed USL in Figure~\ref{fig:learing_curves_sa1} and Figure~\ref{fig:learing_curves_sa2}, respectively. We found that Lagrangian-based methods are susceptible to the learning rate of the Lagrangian multiplier(s) in stochastic primal-dual optimization. First, the oscillating $\lambda$ causes non-negligible deviations in the learning curves. Besides, the increasing $\eta_\lambda$ may degrade the performance dramatically. The phenomenon is especially pronounced in FAC, which has a multiplier network to predict the state-dependent $\lambda(s;\xi)$. Thus, we suggest $\eta_\lambda \ll \eta_\theta$ in practice. As for USL, we find if the penalty factor $\kappa$ is too small, the cost return may fail to converge. Nevertheless, if $\kappa$ is sufficiently large, the learning curves are robust and almost identical. Thus, we suggest $\kappa > 5$ in experiments and a grid search for better performance.

\begin{figure}[H]
      \centering
    \hspace{0.5cm}\includegraphics[width=0.85\linewidth,trim=0 0 0 0,clip]{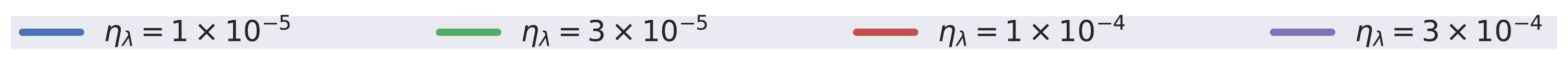}\vspace{0.1cm}\\
    \subcaptionbox{Reward-Lag-SpeedLimit}
        {\includegraphics[width=0.23\linewidth,trim=20 20 0 0,clip]{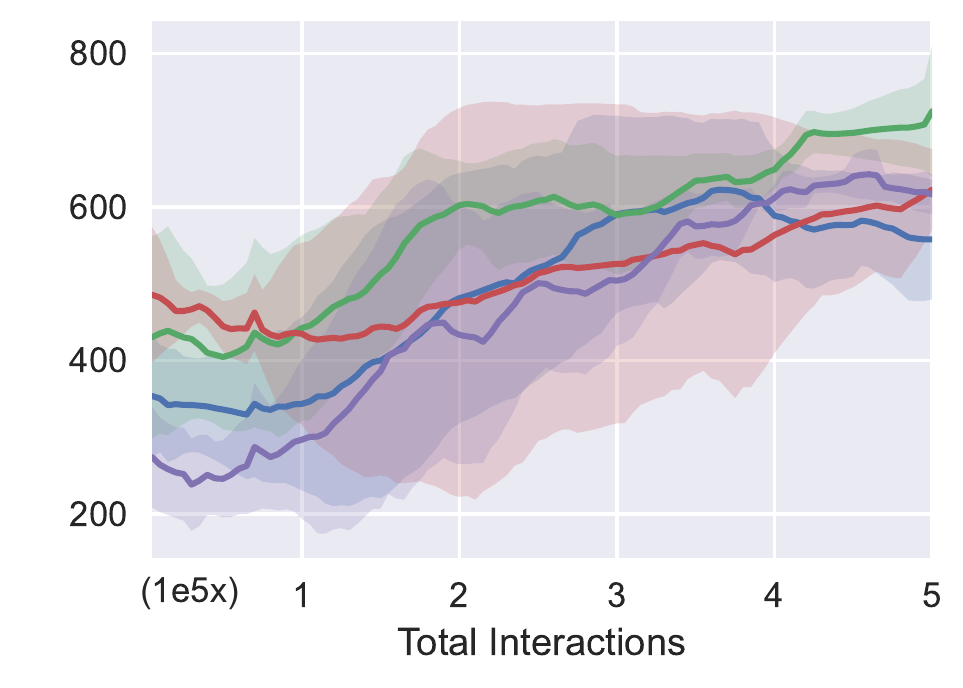}}
       \subcaptionbox{Cost-Lag-SpeedLimit}
        {\includegraphics[width=0.23\linewidth,trim=20 20 0 0,clip]{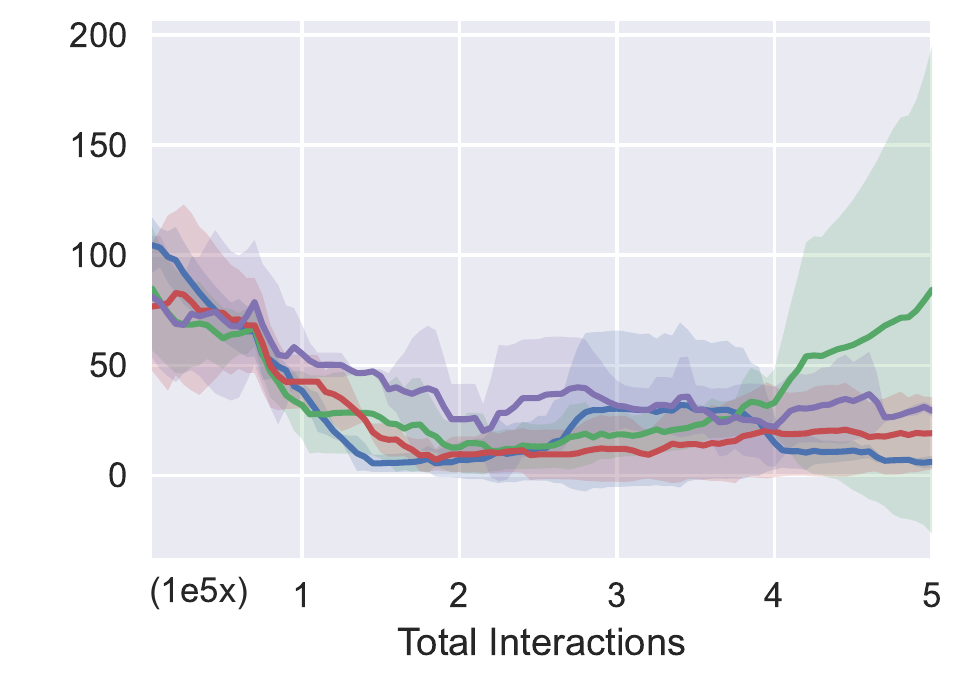}}
      \subcaptionbox{Reward-FAC-MetaDrive}
        {\includegraphics[width=0.23\linewidth,trim=20 20 0 0,clip]{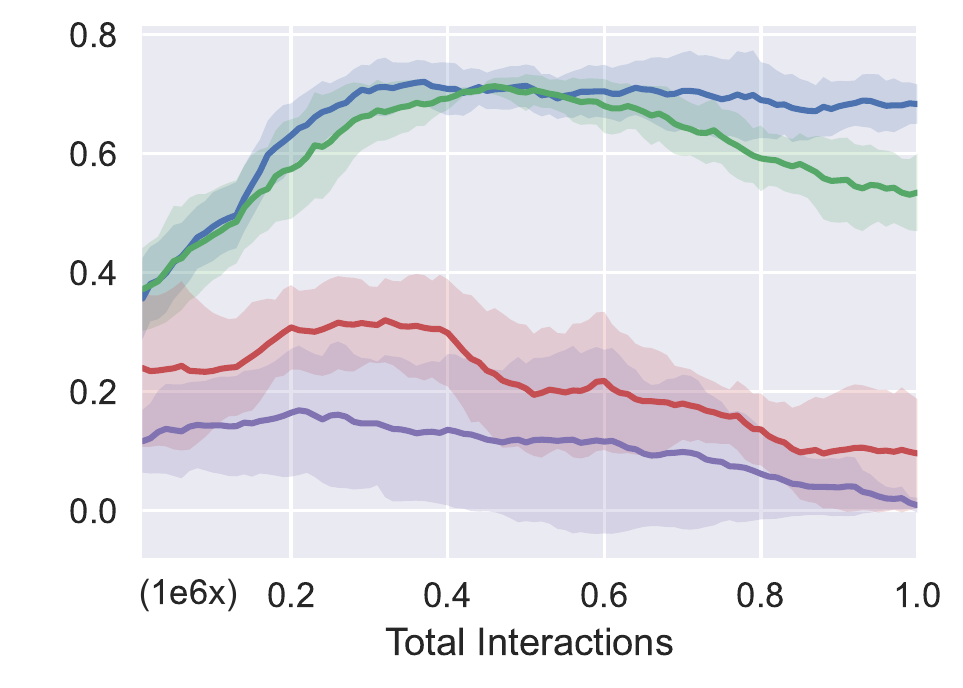}}
    \subcaptionbox{Cost-FAC-MetaDrive}
        {\includegraphics[width=0.23\linewidth,trim=20 20 0 0,clip]{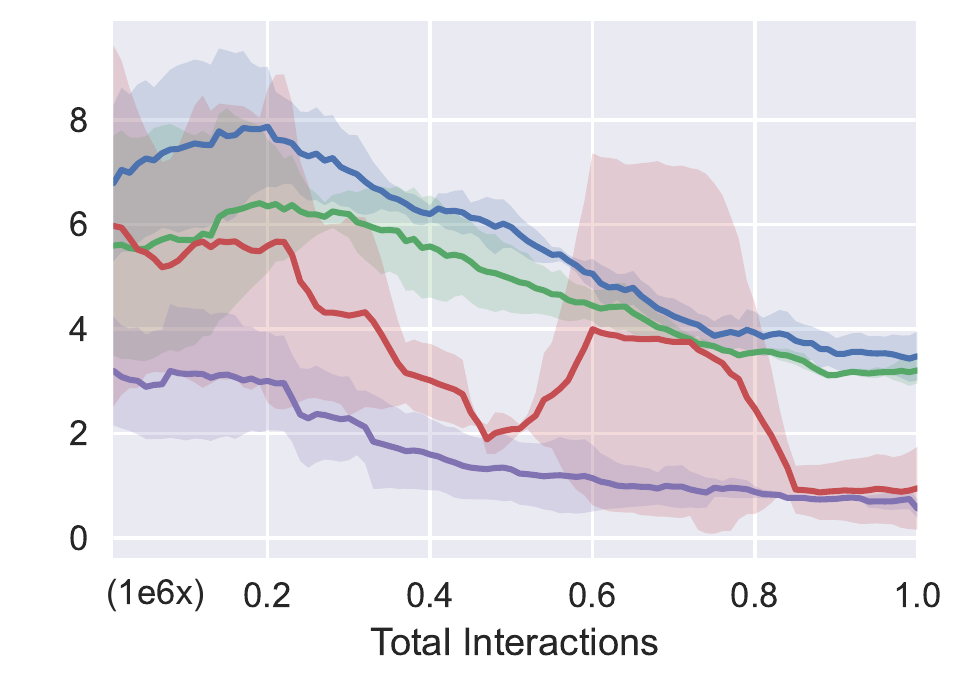}}
      \caption{Sensitivity study of Lagrangian-based methods. The first two figure are reward and cost plots of Off-policy Lagrangian on Car-SpeedLimit task with different $\lambda$ learning rates. The last two figure are success rate and cost plots of Feasible Actor-Critic on MetaDrive benchmark with different $\lambda(s;\xi)$ learning rates.}
      \label{fig:learing_curves_sa1}
      
\end{figure}

\begin{figure}[H]
      \centering
    \hspace{0.5cm}\includegraphics[width=0.85\linewidth,trim=0 0 0 0,clip]{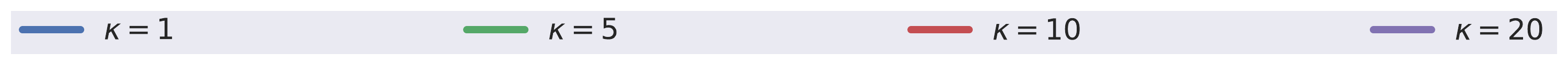}\vspace{0.1cm}\\
 \subcaptionbox{Reward-USL-SpeedLimit}
        {\includegraphics[width=0.23\linewidth,trim=20 20 0 0,clip]{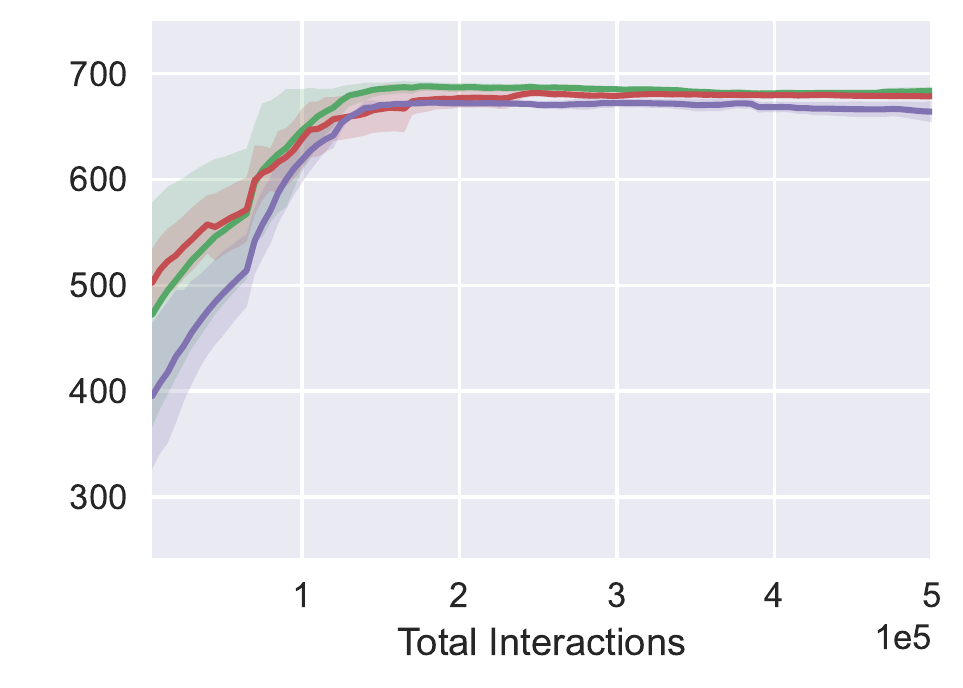}}
       \subcaptionbox{Cost-USL-SpeedLimit}
        {\includegraphics[width=0.23\linewidth,trim=20 20 0 0,clip]{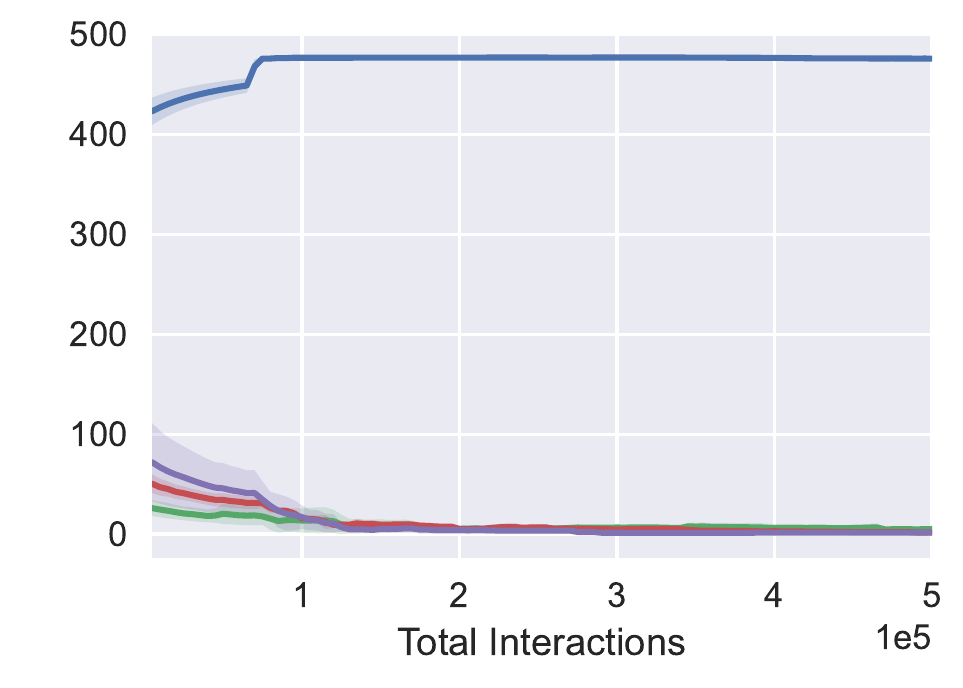}}
      \subcaptionbox{Reward-USL-MetaDrive}
        {\includegraphics[width=0.23\linewidth,trim=20 20 0 0,clip]{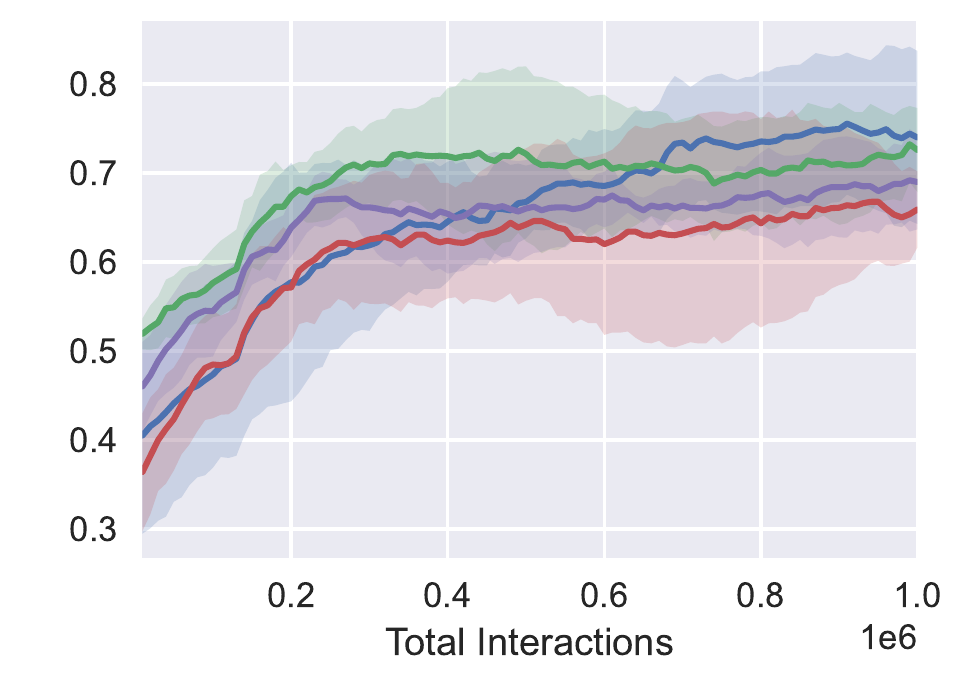}}
    \subcaptionbox{Cost-USL-MetaDrive}
        {\includegraphics[width=0.23\linewidth,trim=20 20 0 0,clip]{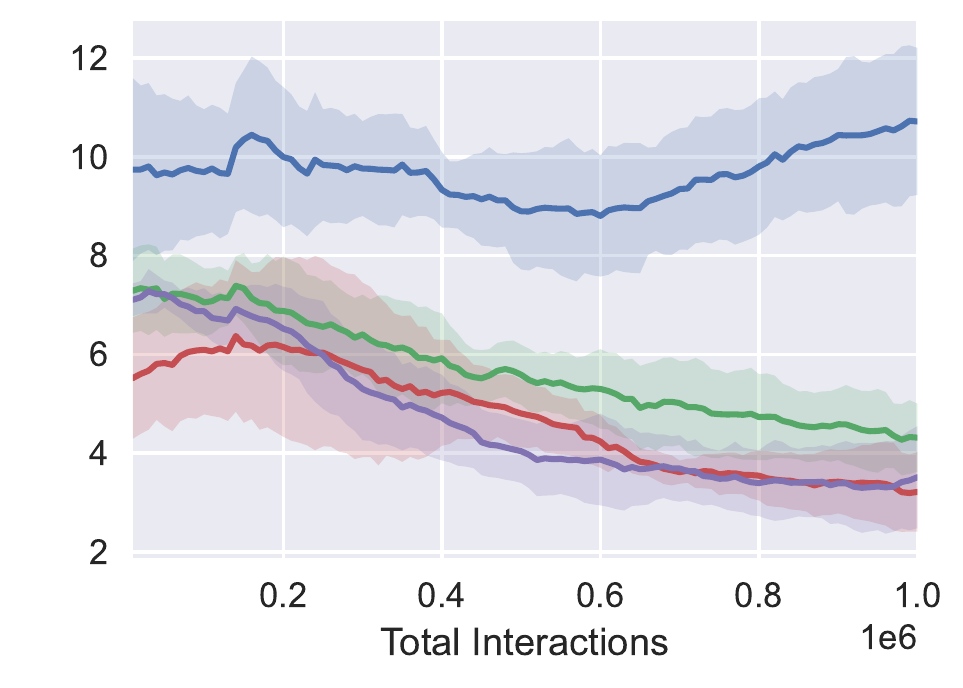}}
      \caption{Sensitivity study of USL. The first two figure are reward and cost plots of USL on the Car-SpeedLimit task with different penalty factors $\kappa$. The last two figure are the success rate and cost plots of USL on the MetaDrive benchmark with different penalty factors $\kappa$.}
      \label{fig:learing_curves_sa2}
\end{figure}

\subsection{C.3 Additional Comparative Study between State-wise Safe RL, Episodic Safe RL and Reward Shaping}

In Section 1: Introduction, we claim that ``Penalizing unsafe transitions on the reward function (i.e., $r' = r-\sigma\cdot c$) is straightforward but sometimes cumbersome to navigate the trade-off between performance and safety.'' In this section, we empirically demonstrate that by changing punishment intensities of the reward shaping method on the Stabilization task.
\begin{figure}[H]
      \centering
    \hspace{0.5cm}\includegraphics[width=0.85\linewidth,trim=0 0 0 0,clip]{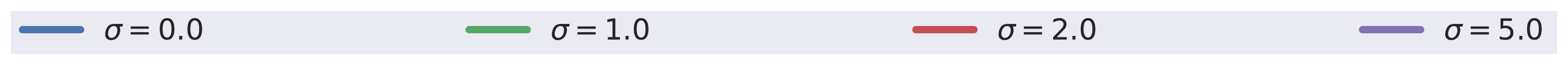}\vspace{0.1cm}\\
 \subcaptionbox{Reward-Stabilization}
        {\includegraphics[width=0.23\linewidth,trim=20 20 0 0,clip]{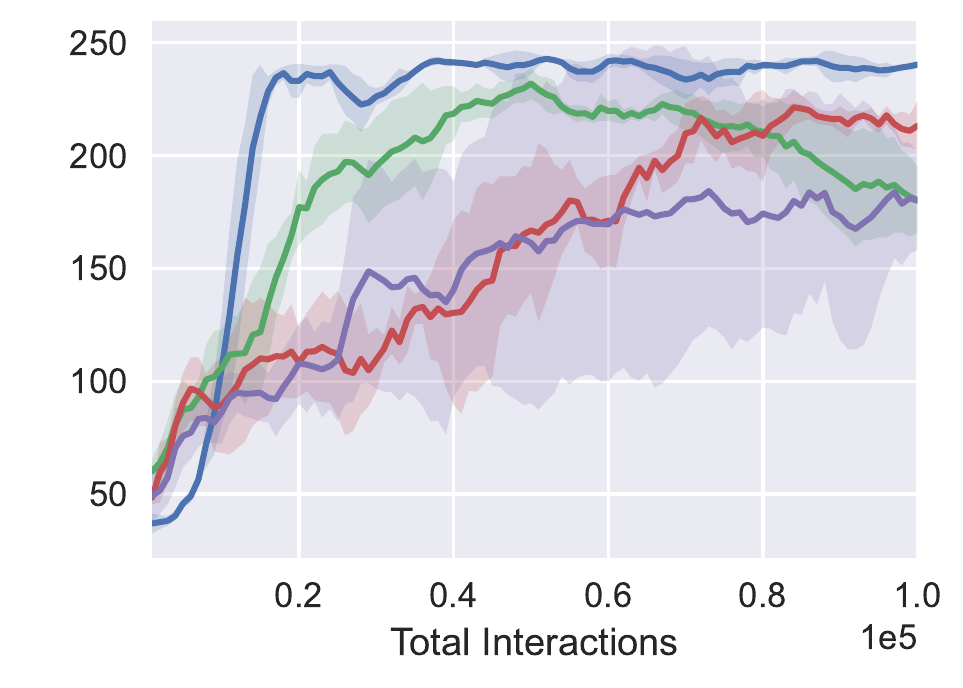}}
       \subcaptionbox{Cost-Stabilization}
        {\includegraphics[width=0.23\linewidth,trim=20 20 0 0,clip]{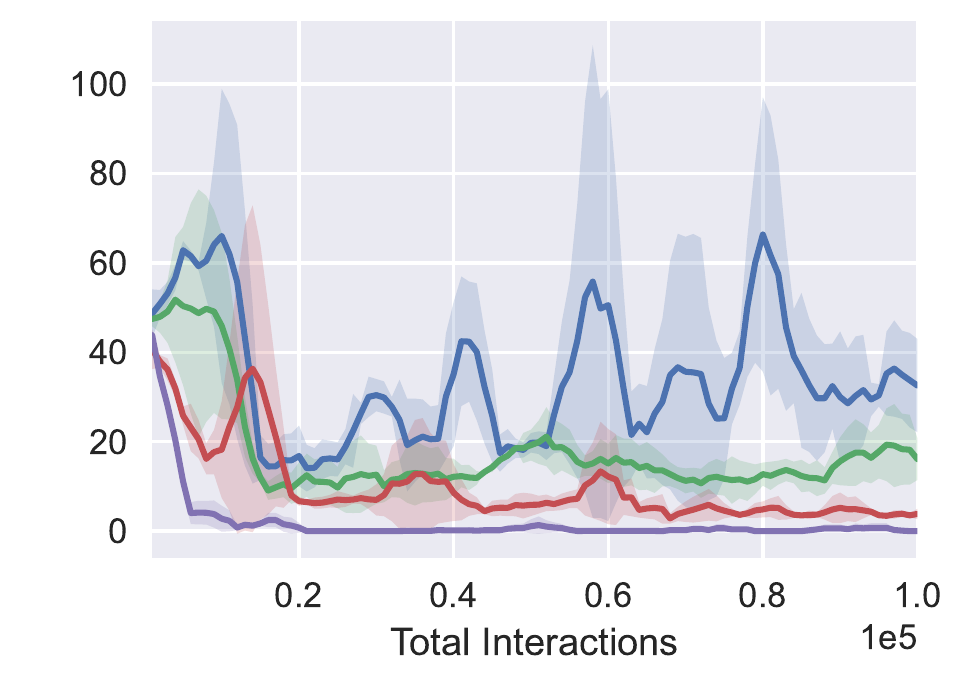}}
      \subcaptionbox{CostRate-Stabilization}
        {\includegraphics[width=0.23\linewidth,trim=20 20 0 0,clip]{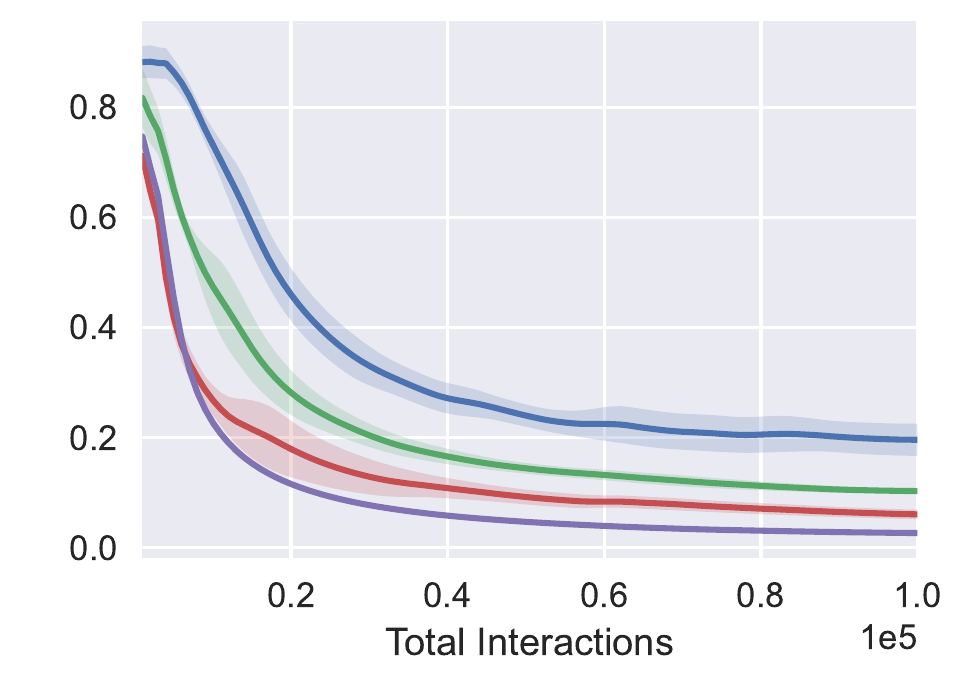}}
      \caption{Comparative study of different punishment intensities of the reward shaping method on the Stabilization task.}
      \label{fig:learing_curves_compare3}
\end{figure}

In Section 3: Revisit RL toward Safety-critical Tasks, we claim that ``In many safety-critical scenarios, the final policy is supposed to maintain the zero-cost return since any inadmissible behavior could lead to catastrophic failure in the execution. Prior constrained learning paradigms have fatal flaws under this premise. For the episode constraint, if we set the threshold $d$ close to $0$, the agent either fails to improve policy or receives a cost-return more significant than $0$." In this section, we empirically demonstrate that episodic safe RL may not work toward safety-critical tasks when we set $d\rightarrow0$ in the corresponding constraint $J_c(\pi) \leq d$. We perform a comparative study on the Stabilization task and take CPO, PPO-L, TRPO-L~\cite{ray2019benchmarking} as episodic baselines. The results show that CPO/PPO-L/TRPO-L fails to improve policy and receives a cost-return more significant than $0$ when we set $d=0$. Instead, the state-wise safe RL method (we take USL as an example) can obtain a reasonable policy while adhering to zero-cost return in the end.

\begin{figure}[H]
      \centering
    \hspace{0.5cm}\includegraphics[width=0.85\linewidth,trim=0 0 0 0,clip]{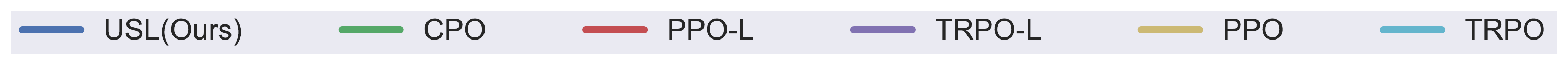}\vspace{0.1cm}\\
 \subcaptionbox{Reward-Stabilization}
        {\includegraphics[width=0.23\linewidth,trim=20 20 0 0,clip]{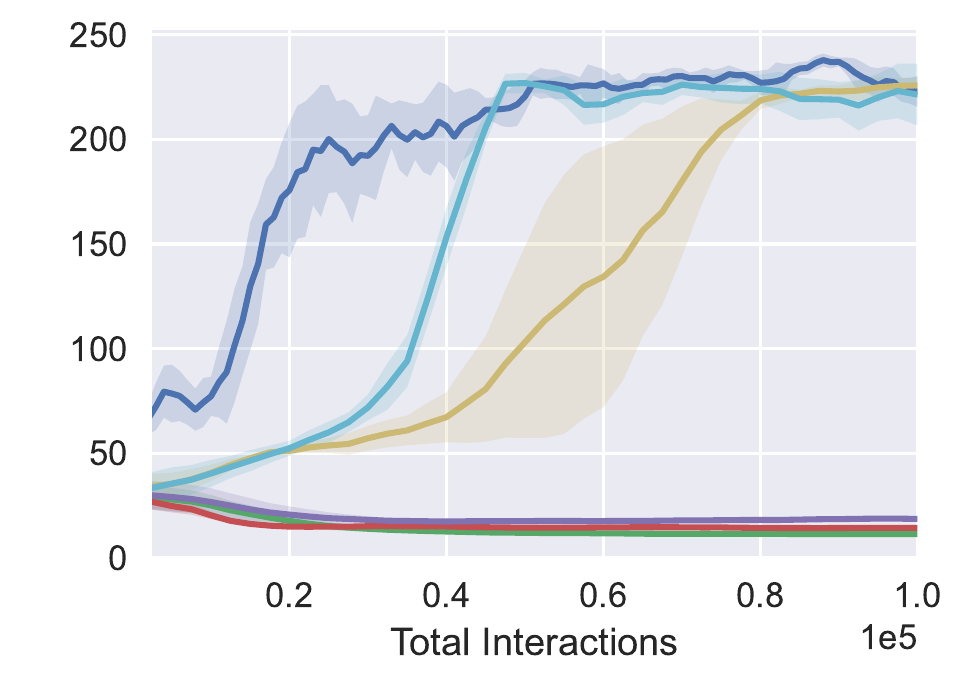}}
       \subcaptionbox{Cost-Stabilization}
        {\includegraphics[width=0.23\linewidth,trim=20 20 0 0,clip]{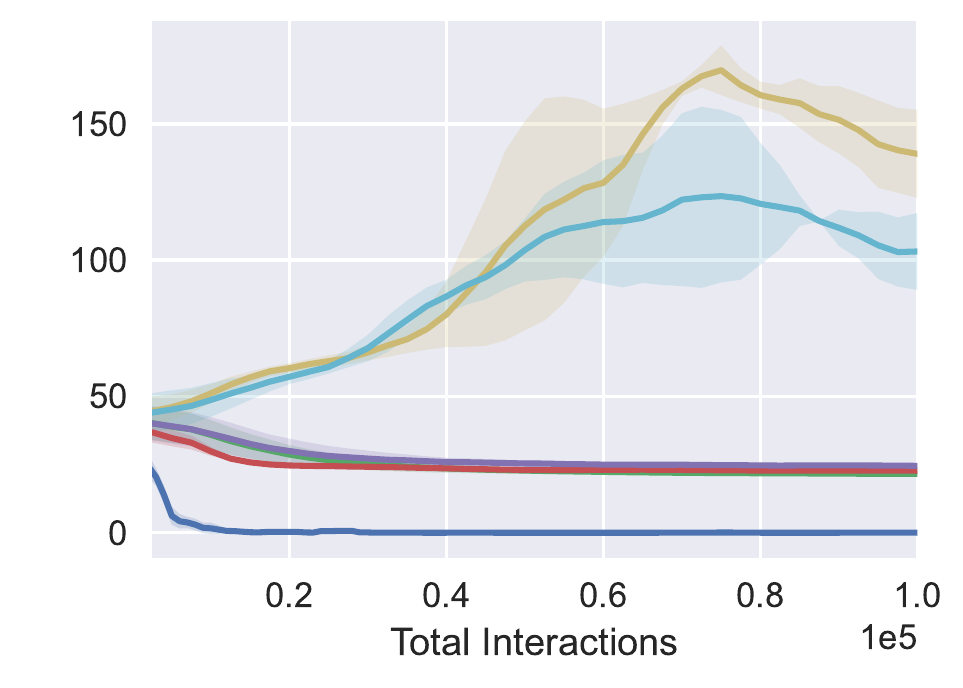}}
      \subcaptionbox{CostRate-Stabilization}
        {\includegraphics[width=0.23\linewidth,trim=20 20 0 0,clip]{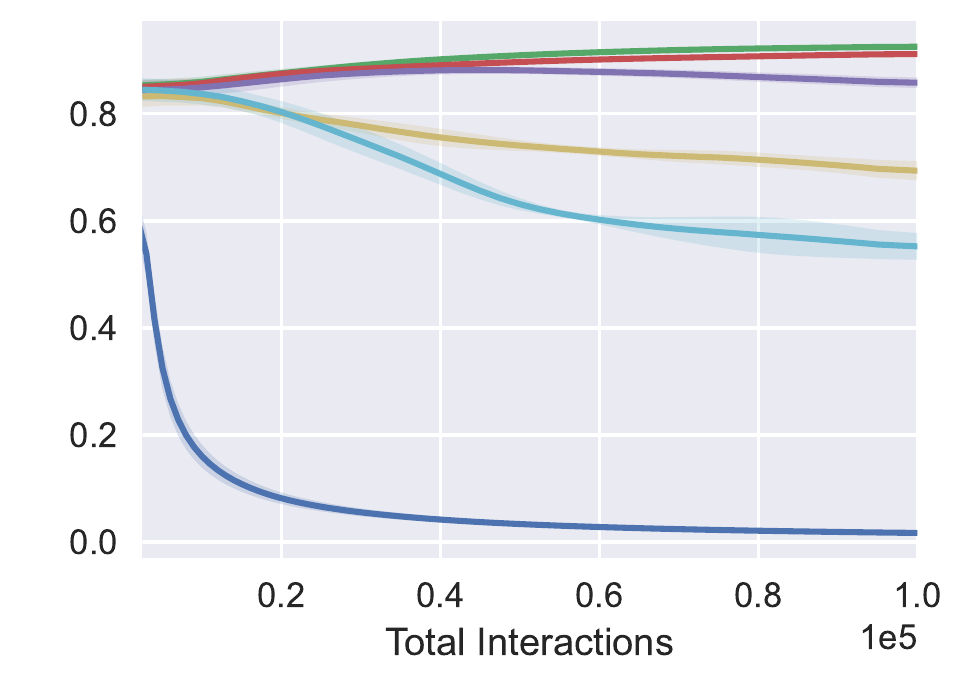}}
      \caption{Comparative study between state-wise safe RL (USL) and episodic safe RL (CPO/PPO-L/TRPO-L) on the Stabilization task.}
      \label{fig:learing_curves_compare1}
\end{figure}

In Section 3: Revisit RL toward Safety-critical Tasks, we claim that ``Empirical results demonstrate safe RL methods adhering to state-wise safety constraints are robust to $\delta$ value.'' In this section, we empirically demonstrate that change $\delta$ from 0.1 to 1 (namely, change safe planning span from 100 to 1) won't affect the learning curves too much. However, when we set $\delta=1$, it degenerates to the instantaneous safety constraint, and the agent cannot obtain a zero-cost return policy at convergence.


\begin{figure}[H]
      \centering
    \hspace{0.5cm}\includegraphics[width=0.85\linewidth,trim=0 0 0 0,clip]{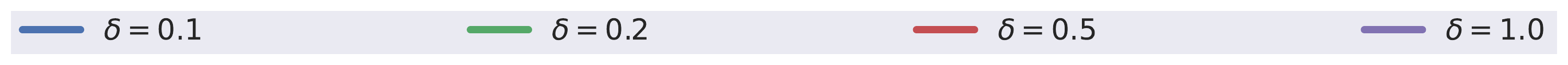}\vspace{0.1cm}\\
 \subcaptionbox{Reward-Stabilization}
        {\includegraphics[width=0.23\linewidth,trim=20 20 0 0,clip]{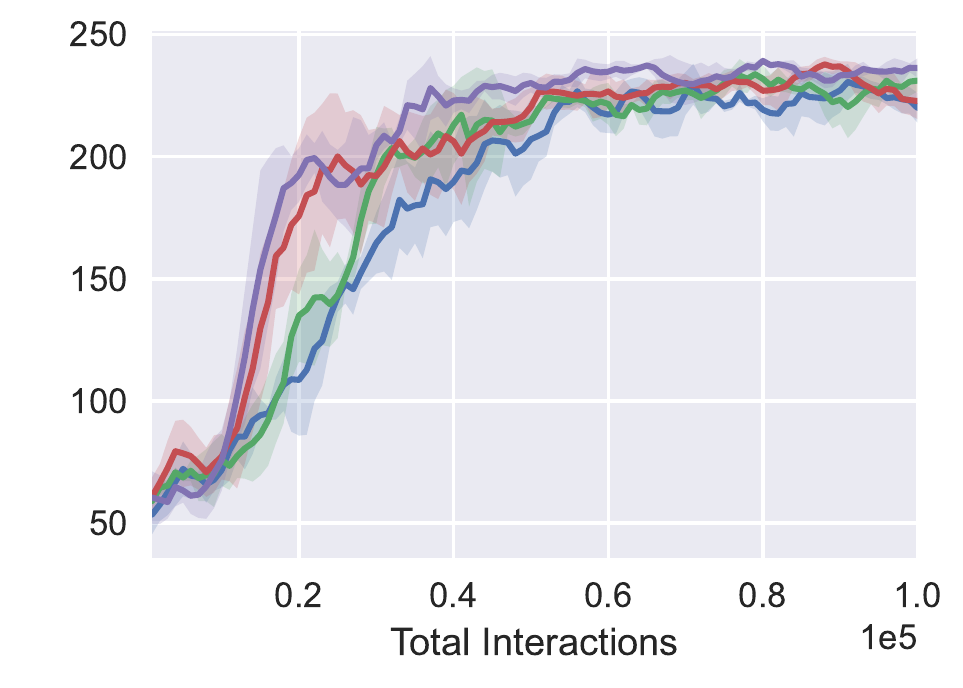}}
       \subcaptionbox{Cost-Stabilization}
        {\includegraphics[width=0.23\linewidth,trim=20 20 0 0,clip]{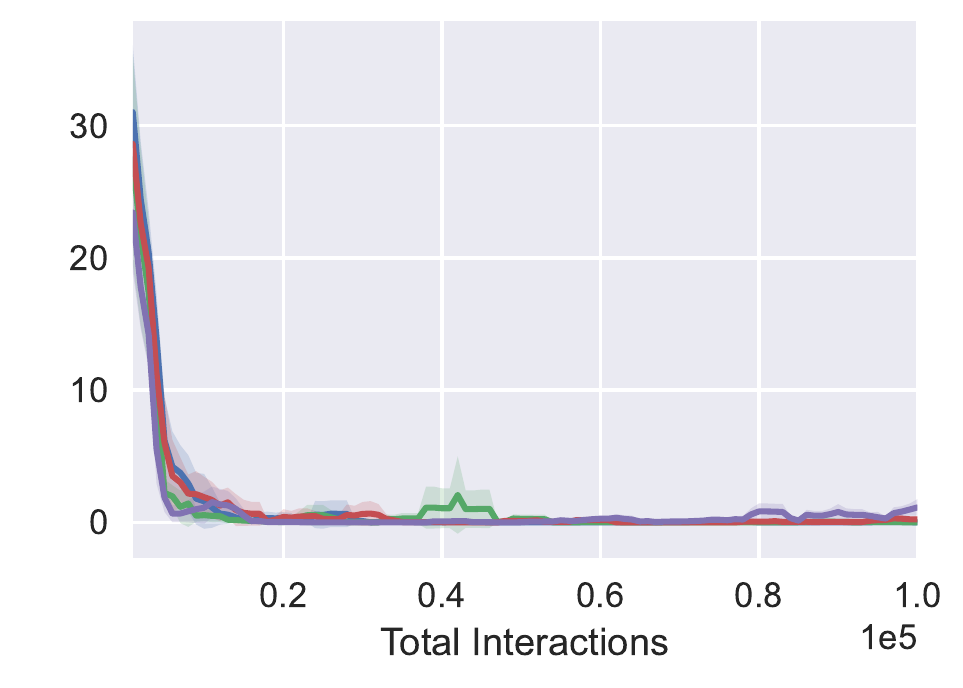}}
      \subcaptionbox{CostRate-Stabilization}
        {\includegraphics[width=0.23\linewidth,trim=20 20 0 0,clip]{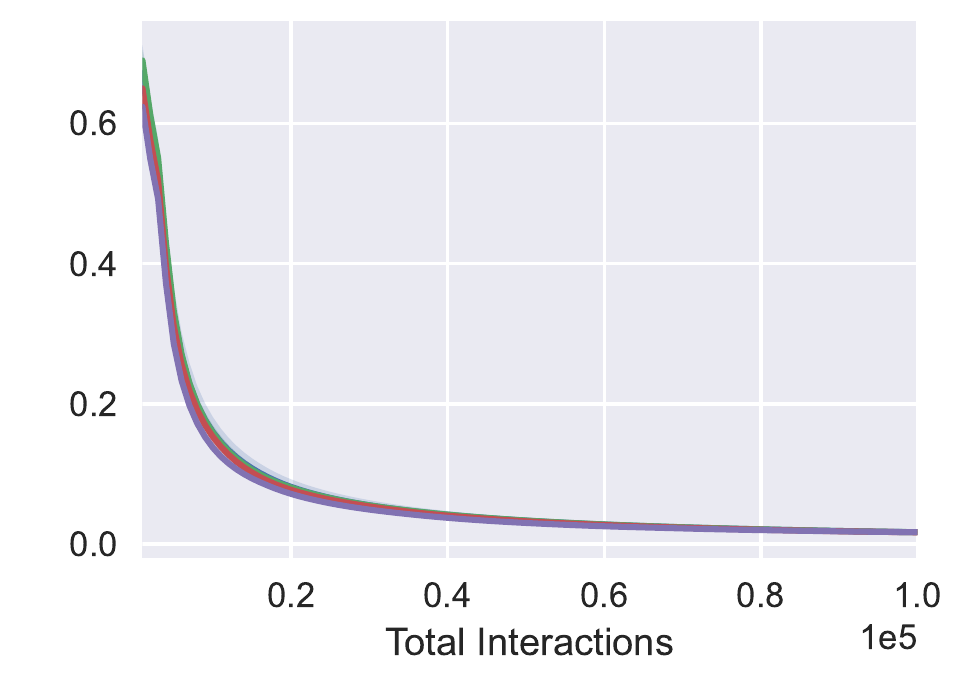}}
      \caption{
      Comparative study of different safe planning span settings in state-wise safe RL (USL) on the Stabilization task.}
      \label{fig:learing_curves_compare2}
\end{figure}

\end{document}